\documentclass{article}
\usepackage{iclr2027_conference,times}

\usepackage{amsmath,amssymb,amsthm,mathtools,bm}
\usepackage{booktabs,tabularx,multirow,array}
\usepackage{graphicx,subcaption}
\usepackage{comment}
\usepackage{placeins}
\usepackage{float}
\usepackage{needspace}
\usepackage{algorithm,algpseudocode}
\usepackage{enumitem}
\usepackage{xcolor}
\usepackage{microtype}
\usepackage{tikz}
\usetikzlibrary{arrows.meta,positioning,fit}
\usepackage{hyperref}
\hypersetup{colorlinks=true,citecolor=blue!55!black,linkcolor=blue!55!black,urlcolor=blue!55!black}
\usepackage{url}
\usepackage{xspace}
\usepackage{kotex}

\newcommand{\pa}{\operatorname{Pa}}
\newcommand{\src}{\operatorname{Src}}
\newcommand{\nd}{\operatorname{Nd}}

\newcommand{\Var}{\operatorname{Var}}
\newcommand{\eps}{\varepsilon}
\newcommand{\E}{\mathbb E}
\newcommand{\Pp}{\mathbb P}

\newtheorem{theorem}{Theorem}
\newtheorem{lemma}{Lemma}
\newtheorem{proposition}{Proposition}
\newtheorem{corollary}{Corollary}
\newtheorem{assumption}{Assumption}
\theoremstyle{definition}
\newtheorem{definition}{Definition}
\theoremstyle{remark}

\setlist[itemize]{leftmargin=1.4em,itemsep=0.1em,topsep=0.25em}
\setlist[enumerate]{leftmargin=1.55em,itemsep=0.1em,topsep=0.25em}
\renewcommand{\arraystretch}{1.08}
\title{Local Fisher Information Enables Sparse Causal Discovery}

\author{%
    \parbox{\dimexpr\textwidth-2\tabcolsep\relax}{\centering
    Byeongguk Kang\textsuperscript{1}\quad
    Donghyeon Lee\textsuperscript{2}\quad
    Euijong Song\textsuperscript{2}\quad
    Gunwoong Park\textsuperscript{1,2,3}\\[0.5em]
    {\normalfont\small
    \textsuperscript{1}\,Interdisciplinary Program in Artificial Intelligence\\
    \textsuperscript{2}\,Department of Statistics\\
    \textsuperscript{3}\,Institute for Data Innovation in Science\\
    Seoul National University, Korea\\[0.25em]
    \texttt{gw.park23@gmail.com}}}%
}

\iclrfinalcopy
\begin{document}\maketitle
\fancyhead{}
\raggedbottom

\begin{abstract}
Sparse causal discovery calls for methods that exploit graph structure without estimating high-dimensional densities. We introduce Fisher Information Completion Search (FiCS), a source-first algorithm for additive noise models that uses one local Fisher score for both ordering and parent selection. Under regularity and nonconstant-parent conditions, we prove that a node's local Fisher information equals the noise Fisher information exactly when the conditioning set contains all parents, provided that it contains no descendants. This Fisher parent completion identifies the parent set as the unique minimal Fisher completion. With a maximum conditioning set size $q$ at least the maximum indegree $d$, population FiCS queries marginals of at most $q+1$ variables and recovers the true directed acyclic graph under a positive ordering margin. Bounded conditioning also has a population advantage: reducing $q$ toward $d$ cannot decrease, and can strictly increase, the ordering margin. A growing non-Gaussian family separates local Fisher selection from conditional-variance and leaf-first Fisher ordering. For the regularized kernel Stein estimator, we establish high-dimensional DAG consistency under
\(q\{1+\log(p/q)\}+\log p=o(n)\), uniform Fisher separation, local approximation, and compatible ridge and parent penalty parameters. Experiments show the strongest gains when $n$ is small relative to $p$, quantify the effect of the conditioning size, and demonstrate competitive reference-graph recovery on three real-data benchmarks.
\end{abstract}

\section{Introduction}

Graph sparsity is central to statistically efficient high-dimensional directed acyclic graph (DAG) learning~\citep{kalisch2007estimating,nandy2018high,aragam2019globally}. For identifiable subclasses of linear additive noise models (ANMs), an indegree bound reduces structure recovery to local regressions or subset comparisons. This yields minimax-optimal dependence on indegree and ambient dimension for equal-variance Gaussian models~\citep{gao2022optimal} and, under separate assumptions, sub-Gaussian non-Gaussian models~\citep{oh2025optimal}. Can sparsity similarly reduce distribution-sensitive recovery of both the order and parent sets in nonlinear ANMs to bounded-dimensional estimation?

Source-first search is naturally aligned with sparsity: a node becomes eligible once its parents have been selected. Conditional-variance criteria extend this strategy beyond linear-Gaussian models~\citep{park2020variance}. Prominent density-derivative methods instead exploit distributional shape through leaf-first queries on the joint distribution of the remaining variables~\citep{rolland2022score,montagna2023das,montagna2023nogam,xu2024caps,chen2026fico}. Graph sparsity alone does not make these ordering queries low-dimensional. We resolve this sparsity--query mismatch through Fisher parent completion, which characterizes the parent set as the minimal Fisher completion for source-first ordering and parent recovery.

Our key result is \emph{Fisher parent completion}: among nondescendant conditioning sets, local Fisher information attains the noise Fisher information if and only if all parents are included. Conditioning on the parents fixes the structural signal; omitting a parent yields a positive parent-omission margin. This yields Fisher Information Completion Search (FiCS), which uses the same local score for source-first ordering and size-penalized parent selection. Node ordering additionally requires a positive ordering margin because noise Fisher information may differ across nodes.

\needspace{4\baselineskip}
Our contributions are threefold.
\begin{itemize}[leftmargin=1.5em,itemsep=0.35em]
\item \textbf{Fisher parent completion.}
We characterize the parent set as the unique minimal Fisher completion, yielding FiCS, which uses one local score for both source-first ordering and parent recovery from marginals of dimension at most $q+1$ when $q\ge d$ and the ordering margin is positive.

\item \textbf{Two benefits of bounded conditioning.} Restricting $q$ limits estimation dimension and, once $q\ge d$, reducing $q$ cannot decrease and can strictly increase the population ordering margin. A growing non-Gaussian family satisfies uniform Fisher separation while the CaPS leaf rule and conditional-variance ordering make incorrect first population selections.

\item \textbf{High-dimensional consistency.}
We establish DAG consistency using a uniform local score error bound and derive the corresponding error bound for the regularized kernel Stein estimator, with complexity
\(q\{1+\log(p/q)\}+\log p\).
Consistency holds under uniform Fisher separation, local approximation, and compatible ridge and parent penalty parameters.
\end{itemize}

\section{Related Work}

\paragraph{Source-first recovery from conditional uncertainty.}
In linear equal-variance models, \citet{chen2019equal} introduced indegree-bounded conditional-variance ordering. \citet[Algorithm~1, LearnDAG]{gao2022optimal} combines this ordering with parent recovery and establishes minimax-optimal sample complexity. Conditional-variance conditions also cover nonlinear ANMs with heterogeneous noise variances~\citep{park2020variance}. For jointly Gaussian variables, FiCS ordering reduces to bounded-subset conditional-variance ordering; Appendix~\ref{app:gaussian-specialization} gives the corresponding score identity and estimation bound. Fisher parent completion extends this local approach to distribution-sensitive recovery in nonlinear ANMs. Unlike CAM's additive regression model~\citep{buhlmann2014cam}, it permits structural functions that are not additive across individual parents.

\paragraph{Leaf-first recovery from density derivatives.}
SCORE, DAS, NoGAM, CaPS, and FICO infer leaf-first orderings from density scores, Hessians, or Fisher information~\citep{rolland2022score,montagna2023das,montagna2023nogam,xu2024caps,chen2026fico}. Under Stein integration conditions, the expected-curvature and Fisher statistics used by CaPS and FICO give equivalent population leaf rankings. Their queries use the remaining-variable distribution, whereas FiCS conditions on selected predecessors to identify completed parent mechanisms. Section~\ref{sec:ordering-regions} compares the resulting ordering rules; existing comparisons of sufficient identifiability conditions address a different question~\citep{park2026comparison}.

\paragraph{Other uses of locality.}
AdaScore extends score-based discovery across linear, nonlinear, and latent-variable models; HiTOC localizes leaf search to Markov-blanket subsets; SSTS reuses score geometry through Schur-complement updates~\citep{montagna2025adascore,sun2026hitoc,wu2026ssts}. FiCS instead bounds each queried marginal by the subset size and uses the same local score for ordering and parent selection. Section~\ref{sec:guarantees} analyzes regularized kernel Stein estimation uniformly over those local queries.

\section{ANM Identifiability via Fisher Parent Completion}
\label{sec:fisher_parent_completion}

\emph{Fisher parent completion} characterizes when a local conditioning set contains all parents. We use this principle to recover an ordering and parent sets, establish the population advantage of locality, and compare the resulting ordering regimes.

\subsection{Fisher Parent Completion in Additive Noise Models}

Let $G=(V,E)$ be a directed acyclic graph on $V=[p]:=\{1,\ldots,p\}$, let $\pa(i)$ denote the parents of node $i$, and let $d=\max_i|\pa(i)|$ be the maximum indegree. Let $\nd(i)$ denote the nondescendants of $i$, excluding $i$. We consider the additive noise model over $G$:
\begin{equation}
    \label{eq:anm}
    X_i=f_i(X_{\pa(i)})+\eps_i,\qquad \text{independent }  \eps_1, \ldots, \eps_p.
\end{equation}
The structural function $f_i$ may be nonlinear and need not be additive across its parent arguments.
We write $\E$, $\Var$, and $\Pp$ for expectation, variance, and probability, respectively.

For the Fisher parent completion result below, we assume that each noise density is everywhere positive and $C^1$, with finite location Fisher information, and that each $f_i$ is continuous and nonconstant in every parent coordinate. The remaining score-projection and integration-by-parts conditions are stated in Appendix~\ref{app:regularity}.

Let $s_{\varepsilon_i}(e):=\partial_e \log p_{\varepsilon_i}(e)$ and
$I_i^\varepsilon:=\mathbb{E}[s_{\varepsilon_i}(\varepsilon_i)^2]$ denote the
noise score and noise Fisher information, respectively. For $i\notin T$, let
$p_{i,T}$ denote the joint density of $(X_i,X_T)$, and define
\begin{equation}
I_i(T)
:=
\mathbb{E}\!\left[
\left\{\partial_{x_i}\log p(X_i\mid X_T)\right\}^2
\right],
\qquad
J_i(T)
:=
-I_i(T)
=
\mathbb{E}\!\left[
\partial_{x_i}^2 \log p_{i,T}(X_i,X_T)
\right].
\label{eq:local-fisher-score}
\end{equation}
We call $I_i(T)$ the local Fisher information and $J_i(T)$ the local Fisher
score. A nondescendant set $T$ satisfying $I_i(T)=I_i^\varepsilon$ is called a
Fisher completion for node $i$.

\begin{theorem}[Fisher parent completion]
\label{thm:parent-completion} 
Consider an ANM in Equation~\eqref{eq:anm}. Under the stated regularity and
nonconstant-parent conditions, for every node $i$ and every
$T\subseteq \operatorname{Nd}(i)$,
\begin{equation}
\Gamma_i(T)
:=
I_i^\varepsilon-I_i(T)
\ge 0,
\qquad
\Gamma_i(T)=0
\iff
\operatorname{Pa}(i)\subseteq T.
\label{eq:fisher-parent-completion}
\end{equation}
Consequently, $\operatorname{Pa}(i)$ is the unique minimal Fisher completion.
\end{theorem}

Fisher parent completion has a simple projection interpretation:
\[
\Gamma_i(T)
=
\mathbb{E}\!\left[
\operatorname{Var}\!\left\{
s_{\varepsilon_i}(\varepsilon_i)\mid X_i,X_T
\right\}
\right]
=
\min_g
\mathbb{E}\!\left[
\left\{
s_{\varepsilon_i}(\varepsilon_i)-g(X_i,X_T)
\right\}^2
\right].
\]
Thus, $\Gamma_i(T)$ is the minimum prediction error of the noise score from
$(X_i,X_T)$. If $T$ contains all parents, then the noise score $s_{\varepsilon_i}(\varepsilon_i)$ is exactly recovered
by $(X_i,X_T)$ and the error is zero; if a parent is omitted, the error is
strictly positive. Figure~\ref{fig:fisher-parent-completion} illustrates this distinction through the conditional density of $X_i$ given $X_T$.

\begin{figure}[!ht]
\centering
\includegraphics[width=0.48\linewidth]{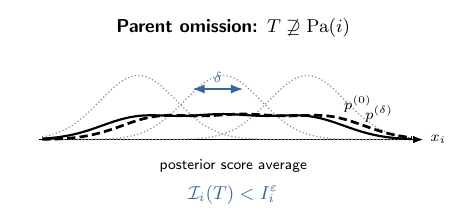}
\hfill
\includegraphics[width=0.48\linewidth]{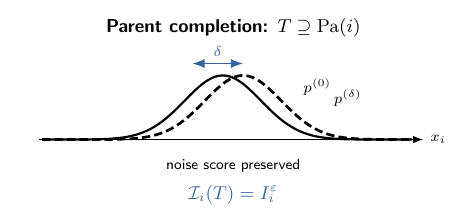}
\caption{
Fisher parent completion.
Black curves show the conditional density of $X_i$ given $X_T$, while gray
curves show translated noise densities $p_{\varepsilon_i}(x_i-m)$ for possible
values of $m=f_i(X_{\mathrm{Pa}(i)})$.
Parent omission averages over multiple translations, giving
$I_i(T)<I_i^\varepsilon$ (left); parent completion leaves a single translation,
giving $I_i(T)=I_i^\varepsilon$ (right).
Solid and dashed black curves denote the conditional density before and after
a small location shift $\delta$.
}
\label{fig:fisher-parent-completion}
\end{figure}

\subsection{Identifiability via Local Fisher Information}
\label{sec:forward-identification}

A topological ordering is a bijection $\pi:V\to[p]$ such that $\pi(i)<\pi(j)$ for every edge $(i,j)\in E$, where $\pi(i)$ denotes the position of node $i$. A set $S\subseteq V$ is \emph{ancestral} if it contains every ancestor of each of its nodes. For an ancestral set $S$ of already selected variables, write $\src(S):=\{i\notin S:\pa(i)\subseteq S\}$ for the current sources.

For an ancestral set $S$ and maximum conditioning set size $q$, let $q_S:=\min\{q,|S|\}$ and define the best local Fisher score
\begin{equation} \label{eq:order-score}
    Q_i^{(q)}(S):=\min_{\substack{T\subseteq S,~|T|=q_S}}J_i(T).
\end{equation}
By monotonicity of local Fisher information under additional conditioning, the minimum over subsets of cardinality $q_S$ in Equation~\eqref{eq:order-score} equals the minimum of $J_i(T)$ over all $T\subseteq S$ with $|T|\le q$. We call the rule that selects a minimizer of $Q_i^{(q)}(S)$ over $i\notin S$ \emph{local Fisher selection}.

Whenever a nonsource remains, define the ordering margin
\begin{equation}
g_q(S):=
\min_{i\notin S\cup\src(S)}Q_i^{(q)}(S)
-
\min_{s\in\src(S)}Q_s^{(q)}(S).
\label{eq:stagewise-gap}
\end{equation}
Differences in noise Fisher information across nodes mean that attainability of a Fisher completion alone does not rank the remaining candidates. The margin $g_q(S)$ measures whether the best attainable score among current sources is separated from that among nonsources.

\begin{theorem}[Sparse population DAG recovery]
\label{thm:population-recovery}
Consider an ANM in Equation~\eqref{eq:anm} satisfying the conditions of Theorem~\ref{thm:parent-completion}, and suppose that $q\ge d$. Then, at every ancestral set $S$, the following statements hold.
\begin{enumerate}[label=(\alph*),leftmargin=1.55em,itemsep=.35em]
    \item \textbf{Ordering.} For every $i\notin S$, we have that $Q_i^{(q)}(S)=-I_i^\eps+\min_{\substack{T\subseteq S,\,|T|=q_S}}\Gamma_i(T)$, and that
    \begin{equation}
        i\in\src(S)\quad\Longleftrightarrow\quad Q_i^{(q)}(S)=-I_i^\eps.
        \label{eq:source-certificate}
    \end{equation}
    If a nonsource remains and $g_q(S)>0$, every minimizer of $Q_i^{(q)}(S)$ over $i\notin S$ is a current source.
    \item \textbf{Parent recovery.} For every $i\in\src(S)$ and every $T\subseteq S$ satisfying $|T|\le q$,
    \begin{equation}
        T=\pa(i)\quad\Longleftrightarrow\quad J_i(T)=Q_i^{(q)}(S)\quad \text{and}\quad J_i(U)>Q_i^{(q)}(S)\ \ \text{for every}\ \ U\subsetneq T.
        \label{eq:parent-certificate}
    \end{equation}
\end{enumerate}
Consequently, if $g_q(S)>0$ at every ancestral set $S$ for which a nonsource remains, repeated application of \textup{(a)} and \textup{(b)} recovers the true DAG using marginals involving at most $q+1$ variables.
\end{theorem}

A transparent sufficient case is equal noise Fisher information across nodes; the same conclusion holds when the noise Fisher information does not increase along directed paths (Corollary~\ref{cor:equal-fisher-recovery}).

\subsection{Population Advantages of Locality}
\label{sec:locality-advantages}

Once $q\ge d$, increasing $q$ leaves current-source scores unchanged but can lower nonsource scores by admitting more proxies for unselected parents.

\begin{proposition}[Population advantage of locality]
\label{prop:locality-benefit}
For $d\le q_1\le q_2$ and every ancestral set $S$,
\begin{equation}
\begin{aligned}
Q_s^{(q_1)}(S)&=Q_s^{(q_2)}(S)=-I_s^\eps &&\text{for every }s\in\src(S),\\
Q_i^{(q_1)}(S)&\ge Q_i^{(q_2)}(S) &&\text{for every }i\notin S\cup\src(S).
\end{aligned}
\label{eq:score-monotonicity}
\end{equation}
Consequently, $g_{q_1}(S)\ge g_{q_2}(S)$ whenever a nonsource remains, and this inequality can be strict.
\end{proposition}

Appendix~\ref{app:subset-population} gives the proof and an example in which $q=d$ yields a positive ordering margin, whereas conditioning on the entire selected set yields a negative margin.

\subsection{Comparison with Other Ordering Criteria}
\label{sec:ordering-regions}

We compare local Fisher selection with the CaPS leaf rule and conditional-variance ordering on a growing non-Gaussian family with uniform Fisher separation.

\begin{definition}[Uniform Fisher separation]
\label{def:uniform-fisher-separation}
A family of ANMs, each equipped with a subset size $q$, is \emph{uniformly Fisher-separated} if there exist constants $\Delta_{\rm ord},\Delta_{\rm par}>0$, common to the entire family, such that, for every model in the family, $g_q(S)\ge\Delta_{\rm ord}$ for every ancestral set $S$ with $V\setminus(S\cup\src(S))\neq\varnothing$, and $\Gamma_i(T)\ge\Delta_{\rm par}$ for every nonroot node $i$ and every parent-omitting set $T\subseteq\nd(i)$ with $|T|\le q$.
Here, $\Delta_{\rm ord}$ and $\Delta_{\rm par}$ are uniform lower bounds for the ordering and parent-omission margins, respectively.
\end{definition}

\begin{proposition}[Uniform Fisher separation in non-Gaussian models]
\label{prop:uniform-nongaussian}
There exists a family of non-Gaussian chain ANMs $\{\mathcal M_p\}_{p\in2\mathbb N}$ with maximum indegree one that is uniformly Fisher-separated at $q=1$. Consequently, Theorem~\ref{thm:population-recovery} guarantees population DAG recovery for every $\mathcal M_p$ using marginals of dimension at most two. In contrast, on every $\mathcal M_p$, the CaPS leaf-selection rule and the conditional-variance ordering rule in \citet[Algorithm~1]{gao2022optimal} choose a nonleaf and a nonsource, respectively, at their first population steps.
\end{proposition}

Appendix~\ref{app:uniform-separation-proof} constructs the family and verifies both margin lower bounds. The conditional-variance comparison evaluates the ordering criterion beyond its equal-variance model class.

\section{Fisher Information Completion Search (FiCS)}
\label{sec:fics}

Suppose that we observe $n$ independent copies $X^{(1)},\ldots,X^{(n)}$ of $X$. FiCS operationalizes the Fisher parent completion principle of Section~\ref{sec:fisher_parent_completion} by replacing each population score $J_i(T)$ with an estimate $\widehat J_i(T)$. We first construct the regularized kernel Stein estimator, then use it for ordering estimation and parent selection, and finally give the computational cost.

\subsection{Local Fisher Score Estimation}
\label{sec:fics-estimation}

We estimate local Fisher scores directly from observations using regularized kernel Stein equations~\citep{zhou2020nonparametric}, extending the density-derivative estimation used by SCORE and CaPS~\citep{rolland2022score,xu2024caps}. The recovery guarantee applies to any estimator with sufficiently small uniform error over the local scores. The theory below analyzes this particular construction.

For $W:=\{i\}\cup T$, let $p_W$ be the density of $X_W$, and define
$s_{W,i}:=\partial_i\log p_W$ and
$r_{W,i}:=(\partial_i^2 p_W)/p_W$.
The factorization $p_W(x_i,x_T)=p(x_i\mid x_T)p_T(x_T)$ yields
\begin{equation}
J_i(T)
=
\E\!\left[\partial_i^2\log p_W(X_W)\right]
=
\E\!\left[r_{W,i}(X_W)-s_{W,i}(X_W)^2\right].
\label{eq:fics-score-identity}
\end{equation}
We estimate the two derivative ratios with separate Stein equations. For observation $a\in[n]$, write $z_a=X_W^{(a)}$. Given a kernel $k_W$, define
\begin{equation}
(K_W)_{ab}:=k_W(z_a,z_b),
\qquad
(\mathbf D_{W,i}^{(m)})_b
:=
\sum_{a=1}^n\partial_i^m k_W(z_a,z_b),
\qquad m\in\{1,2\},
\label{eq:fics-kernel-objects}
\end{equation}
where the derivatives act on the target coordinate of the first kernel argument. The regularized kernel Stein estimator is given by
\begin{equation}
\begin{aligned}
\widehat{\mathbf s}_{W,i}
&:=-(K_W+\eta_W I_n)^{-1}\mathbf D_{W,i}^{(1)},
&
\widehat{\mathbf r}_{W,i}
&:=(K_W+\eta_W I_n)^{-1}\mathbf D_{W,i}^{(2)},\\
\widehat J_i(T)
&:=\frac{1}{n}\mathbf 1^\top\widehat{\mathbf r}_{W,i}
-\frac{1}{n}\|\widehat{\mathbf s}_{W,i}\|_2^2.
\end{aligned}
\label{eq:fics-estimator}
\end{equation}
Here $I_n$ is the $n\times n$ identity matrix, $\mathbf 1$ is the vector of $n$ ones, and $\eta_W>0$ is the matrix ridge that regularizes the matrix inversion. The two Stein systems share the same kernel matrix and can be solved together. Appendix~\ref{app:stein-consistency} derives Equation~\eqref{eq:fics-estimator}, while Appendix~\ref{sec:implementation} specifies the kernels, bandwidths, and numerical safeguards used in the experiments.

\subsection{FiCS Algorithm}

\paragraph{Ordering estimation.}
For the current selected set $S$ and $q_S$ defined in Section~\ref{sec:forward-identification}, FiCS selects the next node by empirical local Fisher selection:
\begin{equation}
\widehat i
\in
\operatorname*{arg\,min}_{i\in V\setminus S}
\;
\min_{\substack{T\subseteq S,~|T|=q_S}}
\widehat J_i(T).
\label{eq:fics-order}
\end{equation}

At the population level, monotonicity of local Fisher information makes minimization over $|T|=q_S$ equivalent to minimization over $|T|\le q$. Estimated scores need not preserve this monotonicity, so Equation~\eqref{eq:fics-order} uses the same cardinality $q_S$ for every candidate. The guarantees in Section~\ref{sec:guarantees} require $q\ge d$ and positive lower bounds on the ordering and parent-omission margins; when $d$ is unknown, $q$ specifies the conditioning-size bound.

\paragraph{Parent selection.}
Given an estimated ordering $\widehat\pi$, let $\widehat S_i:=\{j\in V:\widehat\pi(j)<\widehat\pi(i)\}$ denote the predecessors of node $i$. FiCS estimates its minimal Fisher completion by
\begin{equation}
\widehat P_i
\in
\operatorname*{arg\,min}_{\substack{T\subseteq\widehat S_i,~|T|\le q}}
\left\{
\widehat J_i(T)+\lambda|T|
\right\},
\qquad \lambda>0.
\label{eq:fics-parent}
\end{equation}

Parent selection compares different subset sizes. Given a correct topological order, every parent-set superset has the same population score, so the parent penalty discards redundant predecessors; a parent-omitting set, by contrast, has a positive parent-omission margin. Section~\ref{sec:guarantees} specifies conditions under which $\lambda$ separates these two cases. Restricting parent candidates to ordered predecessors makes the output acyclic, and Algorithm~\ref{alg:fics} summarizes both stages.

\begin{algorithm}[H]
\caption{Fisher Information Completion Search (FiCS)}
\label{alg:fics}
\small

\noindent\textbf{Input:}
Data $\mathbf X\in\mathbb{R}^{n\times p}$, conditioning size $q$, parent penalty $\lambda$, and local scores $\widehat J_i(T)$.

\par\smallskip
\noindent
\begin{minipage}[t]{0.46\linewidth}
\vspace{0pt}
\textbf{Stage 1: Ordering estimation}
\begin{algorithmic}
\State $S\gets\varnothing$
\While{$|S|<p$}
    \State $q_S\gets\min\{q,|S|\}$
    \State $\widehat i\gets$ solution of Eq.~\eqref{eq:fics-order}
    \State $\widehat\pi(\widehat i)\gets |S|+1$
    \State $S\gets S\cup\{\widehat i\}$
\EndWhile
\end{algorithmic}
\end{minipage}
\hfill
\begin{minipage}[t]{0.51\linewidth}
\vspace{0pt}
\textbf{Stage 2: Parent selection}
\begin{algorithmic}
\State $\widehat E\gets\varnothing$
\For{$i\in V$}
    \State $\widehat S_i\gets\{j:\widehat\pi(j)<\widehat\pi(i)\}$
    \State $\widehat P_i\gets$
    solution of Eq.~\eqref{eq:fics-parent}
    \State $\widehat E\gets\widehat E\cup
    \bigl(\widehat P_i
    \times\{i\}\bigr)$
\EndFor
\end{algorithmic}
\end{minipage}

\par\smallskip
\noindent\textbf{Output:}
Ordering $\widehat\pi$ and DAG $\widehat G=(V,\widehat E)$.
\end{algorithm}

Population local Fisher scores are coordinate-scale dependent: replacing $X_i$ by $a_iX_i$, with $a_i\ne0$, multiplies $J_i(T)$ by $a_i^{-2}$. The reported FiCS fits mean-center variables without variance standardization; synthetic and real-data preprocessing are specified in Appendices~\ref{app:synthetic-generation} and~\ref{app:sachs}, respectively.

\subsection{Computational Complexity}
\label{sec:fics-computation}

Each queried marginal has at most $q+1$ variables, while the size of the local query class is
\begin{equation}
N_{p,q}
:=
p\sum_{r=0}^{q}\binom{p-1}{r}
\le p^{q+1}=
O(p^{q+1}).
\label{eq:fics-query-count}
\end{equation}

When several local queries share the same marginal $W$, their Stein systems use the same kernel matrix, so one factorization serves all target coordinates. Thus, at most $\sum_{k=1}^{q+1}\binom{p}{k}$ distinct factorizations are required. Appendix~\ref{sec:implementation} describes the incremental update used to avoid repeated subset comparisons.

A dense factorization costs $O(n^3)$ time and $O(n^2)$ memory. Locality bounds the estimation dimension; exhaustive subset search remains combinatorial in $q$.

\section{Statistical Guarantees of FiCS}
\label{sec:guarantees}

FiCS recovers the DAG when the uniform local-score error is sufficiently small relative to the ordering and parent-omission margins and the parent penalty. We first derive this recovery condition, then bound the uniform local-score error of the regularized kernel Stein estimator and combine the two results. Throughout, $p$, $d$, $q$, the parent penalty $\lambda$, and the ridge parameter $\rho$ may depend on $n$, with this dependence suppressed.

Define the local query class $\mathcal Q_{p,q}$, its size $N_{p,q}$, and the logarithmic complexity term $L$ by
\[
\begin{aligned}
\mathcal Q_{p,q}
&:=
\left\{(i,T):i\in V,\;T\subseteq V\setminus\{i\},\;|T|\le q\right\},
&
N_{p,q}
&:=
|\mathcal Q_{p,q}|,\\
L
&:=
q\left\{1+\log\!\left(\frac{p}{q}\right)\right\}+\log p,
&
\log N_{p,q}
&=
O(L),
\end{aligned}
\]
where $q\{1+\log(p/q)\}:=0$ when $q=0$; the last relation follows from the standard binomial-sum bound applied to Equation~\eqref{eq:fics-query-count}.

\subsection{DAG Recovery under a Uniform Local Score Error Bound}

For $\epsilon>0$, define the uniform score-error event
\begin{equation}
\mathcal E(\epsilon)
:=
\left\{
\max_{(i,T)\in\mathcal Q_{p,q}}
\left|\widehat J_i(T)-J_i(T)\right|
\le\epsilon
\right\}.
\label{eq:uniform}
\end{equation}

\begin{theorem}[DAG recovery from a uniform local score error bound]
\label{thm:finite}
Consider an ANM satisfying the conditions of Theorem~\ref{thm:parent-completion}, and let $d\le q\le p-1$. Suppose $\Delta_{\rm ord},\Delta_{\rm par}>0$ are lower bounds on its ordering and parent-omission margins, respectively. On $\mathcal E(\epsilon)$, Algorithm~\ref{alg:fics} returns a topological ordering and the true DAG whenever
\begin{equation}
2\epsilon
<
\min\left\{
\Delta_{\rm ord},
\lambda,
\Delta_{\rm par}-d\lambda
\right\}.
\label{eq:main-finite}
\end{equation}
\end{theorem}

The three terms in Equation~\eqref{eq:main-finite} preserve source selection, exclude redundant parent supersets, and exclude parent-omitting sets, respectively.

\subsection{Uniform Stein Estimation}

The following condition allows $p>n$ while requiring the logarithmic complexity of the local query class to remain smaller than the sample size.

\begin{assumption}[High-dimensional setting]
\label{ass:high-dimensional}
The maximum indegree and subset size satisfy
\[
\max\{1,d\}\le q\le p-1,
\qquad
L=o(n).
\]
\end{assumption}

For $W=\{i\}\cup T$, let $P_W$ be the law of $X_W$ and let $\mathcal H_W$ be the reproducing kernel Hilbert space (RKHS) of $k_W$. Define the integral operator $\mathsf T_W:L^2(P_W)\to L^2(P_W)$ and the uniform approximation error by
\begin{equation}
\begin{aligned}
(\mathsf T_Wh)(x)
&:=
\E\{k_W(x,X_W)h(X_W)\},\\
A(\rho)
&:=
\sup_{n\ge1}
\max_{(i,T)\in\mathcal Q_{p,q}}
\max_{u\in\{s,r\}}
\left\|
\rho(\mathsf T_W+\rho I)^{-1}u_{W,i}
\right\|_{L^2(P_W)}.
\end{aligned}
\label{eq:l2-approximation}
\end{equation}

\begin{assumption}[Uniform local \(L^2\) Stein regularity]
\label{ass:stein-main}
Uniformly over the model sequence and all $(i,T)\in\mathcal Q_{p,q}$, each $k_W$ is deterministic, positive semidefinite, and $C^4$, with
\[
\sup_x
\left\|
\partial_i^m k_W(x,\cdot)
\right\|_{\mathcal H_W}
\le\kappa_m,
\qquad
m=0,1,2.
\]
The two Stein targets have $L^2(P_W)$ norms bounded by a common constant $M<\infty$, the corresponding first- and second-order integration-by-parts identities hold, and $A(\rho)\to0$ as $\rho\downarrow0$.
\end{assumption}

Assumption~\ref{ass:stein-main} extends regularized kernel Stein estimation~\citep[Section~3.1 and Appendix~B]{zhou2020nonparametric} and $L^2$ Tikhonov analysis~\citep[Sections~2 and~6]{smale2007learning} uniformly over the local queries; it is an estimation condition, not a consequence of bounded indegree. Appendix~\ref{app:stein-assumptions} gives its operator formulation. Uniform Fisher separation is imposed separately through Definition~\ref{def:uniform-fisher-separation}; Propositions~\ref{prop:uniform-nongaussian} and~\ref{prop:gaussian-rbf-example} provide non-Gaussian and Gaussian examples.

For $\delta\in(0,1)$, define
\[
\ell_\delta
:=
\log\!\left(\frac{16N_{p,q}}{\delta}\right),
\qquad
b_\delta
:=
C_B\left\{
\sqrt{\frac{\ell_\delta}{n}}
+
\frac{\ell_\delta}{n}
\right\},
\qquad
e_\delta(\rho)
:=
A(\rho)+\frac{b_\delta}{\rho},
\]
where $C_B$ depends only on the uniform bounds in Assumption~\ref{ass:stein-main}.

\begin{theorem}[Uniform local score error bound]
\label{thm:stein-consistency}
Under Assumption~\ref{ass:stein-main}, set the matrix ridge in Equation~\eqref{eq:fics-estimator} to $\eta_W=n\rho$ for every queried $W$. For every $\delta\in(0,1)$ and every deterministic $\rho\in[b_\delta,1]$, with probability at least $1-\delta$,
\begin{equation}
\max_{(i,T)\in\mathcal Q_{p,q}}
\left|
\widehat J_i(T)-J_i(T)
\right|
\le
C\left\{
e_\delta(\rho)+e_\delta(\rho)^2
\right\},
\label{eq:stein-uniform-rate}
\end{equation}
where $C$ is independent of $n,p,q,\delta,$ and $\rho$.
\end{theorem}

The bound jointly controls the two Stein targets and their same-sample evaluation in the local Fisher functional. Its dependence on $p$ and $q$ enters through $\log N_{p,q}=O(L)$.
\subsection{High-Dimensional Consistency}

Assumption~\ref{ass:high-dimensional} gives $b:=\sqrt{L/n}+L/n\to0$. Choose a deterministic sequence $\rho\downarrow0$ such that $b/\rho\to0$, use the matrix ridge $\eta=n\rho$, and define
\[
e:=A(\rho)+\frac{b}{\rho},\qquad R:=e+e^2,\qquad m:=\min\left\{\Delta_{\rm ord},\lambda,\Delta_{\rm par}-d\lambda\right\}.
\]

\begin{corollary}[High-dimensional DAG consistency]
\label{cor:stein-dag-consistency}
Under Assumptions~\ref{ass:high-dimensional} and~\ref{ass:stein-main}, suppose that the ANM sequence is uniformly Fisher-separated and that every model satisfies the conditions of Theorem~\ref{thm:parent-completion}. If $m>0$ eventually and $R=o(m)$, then the output $\widehat G$ of Algorithm~\ref{alg:fics} satisfies $\Pp(\widehat G=G)\to1$.
\end{corollary}

Theorem~\ref{thm:stein-consistency} gives uniform local score error $O_{\Pp}(R)$, so $R=o(m)$ allows Theorem~\ref{thm:finite} to apply with probability tending to one. A sufficient parent penalty condition is $R=o(\lambda)$ and $d\lambda\le(1-c_0)\Delta_{\rm par}$ for some fixed $c_0\in(0,1)$. Thus $q$ controls $L$, whereas $d$ enters the parent penalty constraint; when $d$ grows, feasibility requires $dR\to0$.

Proofs appear in Appendices~\ref{app:finite-recovery} and~\ref{app:stein-consistency}. Appendix~\ref{app:rkhs-source-result} gives explicit rates and balanced ridge tuning under a stronger RKHS source condition. Empirical bandwidth and synthetic calibration rules are specified in Appendices~\ref{sec:implementation} and~\ref{app:configurations}, respectively.

\section{Numerical Experiments}
\label{sec:experiments}

We test two practical consequences of bounded conditioning: end-to-end recovery when $p$ is large relative to $n$, and sensitivity to the conditioning size. We report directed-edge F1 and ordering accuracy $A_{\mathrm{top}}$, the fraction of true edges consistent with the estimated order.

We generate 30 Erd\H{o}s--R\'enyi (ER3) DAGs for each $p\in\{50,100,200\}$, with expected mean total degree three before capping the indegree at two and randomly permuting variable labels. Each nonroot mechanism is linear with probability $0.5$ and otherwise a Gaussian-process function of its parents; errors are independent $N(0,1)$. We use nested samples with $n\in\{50,100,200,500,1000\}$. We compare FiCS ($q=2$) with CaPS~\citep{xu2024caps}, SCORE~\citep{rolland2022score}, NoGAM~\citep{montagna2023nogam}, and nonlinear DAGMA~\citep{bello2022dagma}. Each method uses one configuration selected on independent calibration DAGs and fixed across test settings. Appendices~\ref{app:exp-details}--\ref{app:additional-experiments} provide implementation details, experimental designs, and additional results. CaPS can fail numerically during GAM pruning. Figures~\ref{fig:exp-main-synthetic} and~\ref{fig:exp-conditioning} include only valid DAG estimates; Appendix~\ref{app:exp-computation} reports the failure counts.

\subsection{Recovery on Random DAGs}

\begin{figure}[H]
\centering
\includegraphics[width=\linewidth]{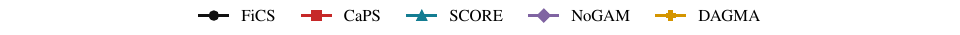}\par
\includegraphics[width=0.32\linewidth]{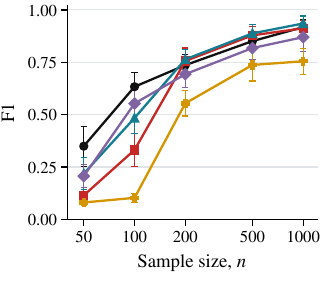}\hfill
\includegraphics[width=0.32\linewidth]{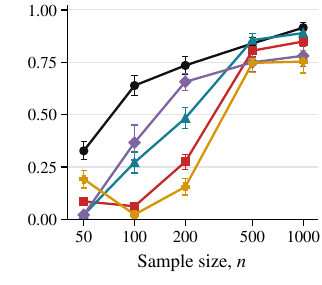}\hfill
\includegraphics[width=0.32\linewidth]{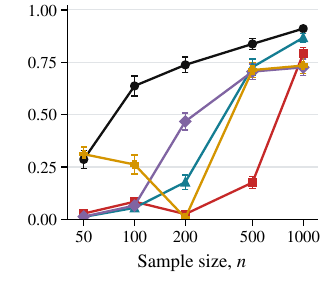}
\caption{Mean graph-recovery F1 versus sample size for $p=50$, $100$, and $200$ from left to right; error bars indicate one standard deviation.}
\label{fig:exp-main-synthetic}
\end{figure}

Figure~\ref{fig:exp-main-synthetic} shows FiCS's largest relative gains in the considered small-sample, high-dimensional settings. At $p=200$ and $n=100$, FiCS attains mean F1 $0.637$, compared with $0.262$ for DAGMA, $0.055$ for SCORE, and $0.066$ for NoGAM. Hence, FiCS recovers substantial graph structure even when $p>n$, supporting the benefit of bounding the dimension of local conditioning. Across the three dimensions, the performance gap generally narrows as $n$ increases, suggesting that FiCS's main empirical advantage in these experiments lies in sample-efficient recovery when data are limited rather than a performance advantage that persists across sample sizes.

\subsection{Sensitivity to the Conditioning Size}

This ablation examines the effect of setting the conditioning size for ordering below, at, or above the maximum indegree. 
At $n=100$ and $p\in\{50,100,150,200\}$, we vary only FiCS's conditioning size for ordering, using $q\in\{1,2,3\}$ or the entire selected set, while fixing the parent-selection subset size at two and all other components. Because the design caps indegree at two, $q=1$ can omit a parent at two-parent nodes and therefore falls outside the uniform DAG recovery guarantee, whereas $q\ge 2$ satisfies the indegree condition. 
CaPS is shown as an external reference.

\begin{figure}[H]
\centering
\begin{minipage}[c]{0.375\linewidth}
\centering
\includegraphics[width=\linewidth]{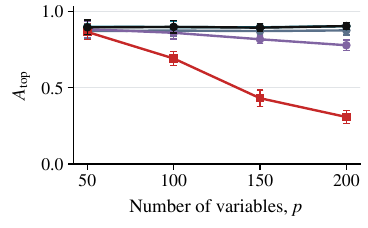}
\end{minipage}\hfill
\begin{minipage}[c]{0.375\linewidth}
\centering
\includegraphics[width=\linewidth]{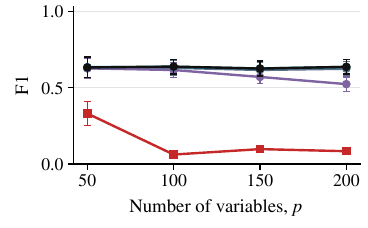}
\end{minipage}\hfill
\begin{minipage}[c]{0.20\linewidth}
\centering
\includegraphics[width=\linewidth]{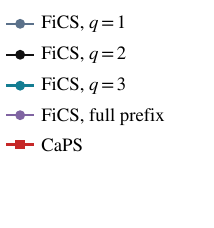}
\end{minipage}
\caption{Mean ordering accuracy (left) and graph-recovery F1 (right) versus dimension at $n=100$.}
\label{fig:exp-conditioning}
\end{figure}

Figure~\ref{fig:exp-conditioning} reveals asymmetric effects of the conditioning size. At $p=200$, $q=1$ attains $A_{\mathrm{top}}=0.876$ and F1 $0.625$, only modestly below $0.904$ and $0.637$ for $q=2$. This small empirical loss is consistent with $q=1$ remaining adequate for roots and one-parent nodes, but it does not restore the DAG recovery guarantee. In contrast, conditioning on the entire selected set yields lower values of $0.779$ and $0.524$, while $q=3$ performs similarly to $q=2$. 
The performance difference widens over the displayed dimensions, consistent with the increasing estimation burden of conditioning on the entire selected set. Because its dimension grows with $p$ whereas fixed-$q$ conditioning remains bounded, this trend suggests a larger disadvantage at still higher dimensions, although those settings are not simulated.
Proposition~\ref{prop:locality-benefit} shows that locality can also improve the population ordering margin for $q\ge d$.

Appendices~\ref{app:exp-graph-results}--\ref{app:linear-fraction} assess sensitivity to graph structure, nonlinear mechanisms, noise distributions, and the linear-mechanism proportion. Appendix~\ref{app:uumc-sortability} compares population and finite-sample ordering with VarSort and $R^2$Sort across six Gaussian generator regimes.

\section{Real-Data Benchmarks} \label{sec:real-data} 

We evaluate whether FiCS ordering transfers beyond synthetic designs and supports reference-graph recovery. Table~\ref{tab:sachs-main} compares Danube~\citep{gnecco2021heavy}, Sachs~\citep{sachs2005}, and 30 CausalRivers subnetworks~\citep{stein2025causalrivers}, using one configuration per method across datasets and $q=2$ for FiCS. We compare the same five methods as in Section~\ref{sec:experiments}. Because the river data are temporally dependent, these experiments assess reference-graph agreement outside the i.i.d.\ ANM setting.

\begin{table}[H]
\centering\footnotesize
\caption{Reference-graph recovery on Danube ($n=500,p=12$), Sachs ($n=853,p=11$), and CausalRivers ($n=500,p=5$). Parentheses give SD over 30 subgraphs. Bold marks the best reported mean per metric; -- denotes no separate ordering stage.}
\label{tab:sachs-main}
\setlength{\tabcolsep}{2pt}
\begin{tabular*}{\linewidth}{@{\extracolsep{\fill}}lrrrrrrrrr@{}}
\toprule
 & \multicolumn{3}{c}{Danube} & \multicolumn{3}{c}{Sachs} & \multicolumn{3}{c}{CausalRivers}\\
\cmidrule(lr){2-4}\cmidrule(lr){5-7}\cmidrule(l){8-10}
Method & $A_{\mathrm{top}}\uparrow$ & F1$\uparrow$ & SHD$\downarrow$ & $A_{\mathrm{top}}\uparrow$ & F1$\uparrow$ & SHD$\downarrow$ & $A_{\mathrm{top}}\uparrow$ & F1$\uparrow$ & SHD$\downarrow$\\
\midrule
FiCS & \textbf{1.000} & \textbf{0.452} & 17 & \textbf{0.529} & \textbf{0.538} & \textbf{11} & \textbf{0.975} (0.076) & \textbf{0.563} (0.165) & 4.53 (1.63)\\ 
CaPS & 0.909 & 0.351 & 36 & \textbf{0.529} & 0.500 & \textbf{11} & 0.733 (0.314) & 0.408 (0.166) & 6.03 (1.43)\\
SCORE & \textbf{1.000} & 0.400 & 33 & \textbf{0.529} & 0.444 & 12 & 0.717 (0.292) & 0.401 (0.175) & 6.20 (1.40)\\ 
NoGAM & 0.727 & 0.225 & 52 & \textbf{0.529} & 0.452 & 14 & 0.742 (0.304) & 0.416 (0.188) & 6.83 (1.66)\\ 
DAGMA & -- & 0.000 & \textbf{11} & -- & 0.000 & 17 & -- & 0.286 (0.337) & \textbf{3.53} (1.25)\\
\bottomrule
\end{tabular*}
\end{table}

Table~\ref{tab:sachs-main} shows that FiCS achieves the highest F1 on Danube and Sachs and the best or tied-best ordering accuracy across all three benchmarks. On Danube, FiCS orders all 11 reference edges forward but selects seven true positives and 13 false positives, showing that correct ordering does not ensure accurate parent selection. DAGMA's lower SHD on Danube and CausalRivers accompanies empty or near-empty graphs with lower F1, so SHD should be interpreted alongside edge recovery.

Appendix~\ref{app:sachs} describes the real-data benchmarks, and Appendix~\ref{app:realdata-sensitivity} reports the sensitivity analysis. Overall, while local Fisher information is especially advantageous in sparse, high-dimensional settings, these results suggest that it also supports effective recovery on low-dimensional benchmarks.

\section{Discussion}

FiCS shows that sparse causal discovery can exploit distributional information without estimating full-dimensional scores. Fisher parent completion characterizes each parent set as the unique minimal Fisher completion, enabling one bounded-dimensional score to drive both source-first ordering and parent recovery. The theory and experiments further show that locality is not merely computational: it can improve the population ordering margin and finite-sample accuracy.

Its DAG recovery guarantee requires $q\ge d$; the method remains sensitive to coordinate scaling and relies on subset search and dense kernel solves. The estimator-agnostic recovery result suggests a clear path forward: combine screening that preserves minimal Fisher completions with faster local score estimators and data-driven, scale-aware tuning. More broadly, Fisher completions may provide a reusable principle for distribution-sensitive causal discovery beyond additive noise models.

\clearpage

\section*{Reproducibility Statement}
Synthetic tuning and test DAGs are generated independently, and all methods are evaluated on the same datasets within each setting. The appendices provide the assumptions and proofs; implementation and hyperparameter selection, experimental design and evaluation metrics, and additional results are detailed in Appendices~\ref{app:exp-details}, \ref{app:experimental-design}, and \ref{app:additional-experiments}, respectively.

\section*{AI Use Statement}
We used ChatGPT (OpenAI) to assist literature retrieval, manuscript drafting and editing, experiment implementation and analysis, and mathematical derivations and proofs. The authors are responsible for verifying the cited sources, code, results, and arguments, and for all claims and content in the final paper.

\clearpage
\bibliography{references}
\bibliographystyle{iclr2027_conference}
\clearpage
\appendix

\section{Proofs of the Main Results}
\label{app:proofs}

This section proves the population and statistical results in the order in which they appear in the main text. Appendix~\ref{app:stein-consistency} also proves the query-uniform concentration lemma used by both the main $L^2$ analysis and Appendix~\ref{app:rkhs-source-result}. Throughout, the notation agrees with Sections~\ref{sec:fisher_parent_completion}--\ref{sec:guarantees}; in particular, $g_q(S)$ and $\Gamma_i(T)$ denote the ordering margin and parent-omission margin, respectively.

\subsection{Proof of Theorem~\ref{thm:parent-completion}}
\label{app:strict-fisher-loss}
\label{app:regularity}

For each local query $(i,T)$, assume that $p(x_i\mid x_T)$ is twice continuously differentiable in $x_i$, its conditional score is square-integrable, and differentiation may pass through the conditional integrals below. Assume that the first and second derivatives are integrable where used and that the boundary terms in integration by parts vanish. Then $\E[\partial_{x_i}^2\log p(X_i\mid X_T)]=-\E[(\partial_{x_i}\log p(X_i\mid X_T))^2]$. These are the score-projection and integration-by-parts conditions used in Section~\ref{sec:fisher_parent_completion}.

\begin{proof}[Proof of Theorem~\ref{thm:parent-completion}]
Fix $i$ and $T\subseteq\nd(i)$, put $Z:=X_T$, $M:=f_i(X_{\pa(i)})$, and write $p_\eps$ and $s_\eps$ for the density and score of $\eps_i$. Since $T$ contains no descendant of $i$, $\eps_i$ is independent of $(M,Z)$. Let $\mu_z$ be a regular conditional distribution of $M$ given $Z=z$. Then
\[ p(x\mid z)=\int p_\eps(x-m)\,\mu_z(dm). \]
Differentiating under the integral and applying Bayes' formula gives, for almost every $(x,z)$,
\[ \partial_x\log p(x\mid z)=\int s_\eps(x-m)\,\pi_{x,z}(dm),\qquad \pi_{x,z}(dm):=\frac{p_\eps(x-m)\mu_z(dm)}{\int p_\eps(x-v)\mu_z(dv)}. \]
Because $p_\eps$ is everywhere positive, $\pi_{x,z}$ and $\mu_z$ are equivalent measures. Evaluating at $x=X_i=M+\eps_i$ yields
\[ \partial_{x_i}\log p(X_i\mid Z)=\E\{s_\eps(\eps_i)\mid X_i,Z\}. \]
The conditional variance decomposition therefore gives
\[ \Gamma_i(T)=I_i^\eps-I_i(T)=\E\!\left[\Var\{s_\eps(\eps_i)\mid X_i,Z\}\right]\ge0, \]
with equality if and only if $s_\eps(\eps_i)$ is $\sigma(X_i,Z)$-measurable.

We next characterize the equality case. If $M$ is $\sigma(Z)$-measurable, then $\eps_i=X_i-M$ is $\sigma(X_i,Z)$-measurable, so $\Gamma_i(T)=0$. Conversely, suppose that $\Gamma_i(T)=0$. For $P_Z$-almost every $z$, the conditional variance under $\pi_{x,z}$ is zero for $P_{X_i\mid Z=z}$-almost every $x$. The mixture density $p(x\mid z)$ is everywhere positive, so this holds for Lebesgue-almost every $x$. Equivalence of $\pi_{x,z}$ and $\mu_z$, followed by Fubini's theorem, gives
\[ s_\eps(x-m)=s_\eps(x-m') \]
for $\mu_z\otimes\mu_z$-almost every $(m,m')$ and Lebesgue-almost every $x$. If $\mu_z$ were not a point mass, two independent draws from $\mu_z$ would differ with positive probability; hence there would exist $m\neq m'$ for which the displayed equality holds for almost every $x$. With $t:=m-m'\neq0$, this implies $s_\eps(y+t)=s_\eps(y)$ almost everywhere. Since $p_\eps$ is positive and $C^1$, the function $y\mapsto\log p_\eps(y+t)-\log p_\eps(y)$ is continuously differentiable with derivative zero almost everywhere and hence everywhere. It is therefore constant, so $p_\eps(y+t)=C p_\eps(y)$ for some $C>0$. Integrating both sides over $y$ forces $C=1$, which would make the positive density $p_\eps$ periodic with nonzero period and therefore nonintegrable. This contradiction shows that $\mu_z$ is a point mass for $P_Z$-almost every $z$. Because $M$ is real-valued and regular conditional laws may be chosen as measurable probability kernels, this is equivalent to the existence of a measurable $m_0$ such that $M=m_0(Z)$ almost surely. Thus
\[ \Gamma_i(T)=0 \quad\Longleftrightarrow\quad M\text{ is }\sigma(Z)\text{-measurable}. \]

It remains to determine when the structural shift is measurable from $X_T$. If $\pa(i)\subseteq T$, then $M$ is directly determined by $X_T$. Conversely, suppose that $T$ omits a parent and write
\[ A:=\pa(i)\cap T,\qquad B:=\pa(i)\setminus T,\qquad C:=T\setminus\pa(i). \]
Then $B\neq\varnothing$, $X_{\pa(i)}=(X_A,X_B)$, and $X_T=(X_A,X_C)$. If $M$ were $\sigma(X_T)$-measurable, there would exist a measurable function $h$ such that
\[ f_i(X_A,X_B)=h(X_A,X_C)\quad\text{a.s.} \]
The positive noise densities imply that $(X_A,X_B,X_C)$ has an everywhere positive joint density, so this equality holds for Lebesgue-almost every $(a,b,c)$. By Fubini's theorem, for almost every $a$ and almost every $c$, the equality holds for almost every $b$. Fix such an $a$ and choose one $c$ from the corresponding full-measure set. Then $f_i(a,b)=h(a,c)$ for almost every $b$, and continuity in $b$ makes $f_i(a,b)$ constant for every $b$. The full-measure set of such $a$ is dense. Hence, for arbitrary $a$ and any $b,b'$, choose a sequence of good points $a_k\to a$; continuity gives $f_i(a,b)=\lim_k f_i(a_k,b)=\lim_k f_i(a_k,b')=f_i(a,b')$. Thus $f_i$ is independent of all coordinates in $B$, contradicting its nonconstant dependence on every parent coordinate. Hence
\[ \Gamma_i(T)=0 \quad\Longleftrightarrow\quad \pa(i)\subseteq T. \]
The Fisher completions are therefore precisely the nondescendant supersets of $\pa(i)$, whose unique inclusion-minimal and unique minimum-cardinality member is $\pa(i)$.
\end{proof}

\subsection{Proof of Theorem~\ref{thm:population-recovery} and Additional Corollary}
\label{app:population-recovery}

\begin{proof}[Proof of Theorem~\ref{thm:population-recovery}]
Fix an ancestral set $S$. Because every $T\subseteq S$ contains only nondescendants of an unselected node $i$, Theorem~\ref{thm:parent-completion} gives $J_i(T)=-I_i^\eps+\Gamma_i(T)$ and hence
\[ Q_i^{(q)}(S) = -I_i^\eps+ \min_{\substack{T\subseteq S,\,|T|=q_S}}\Gamma_i(T). \]
If $i\in\src(S)$, then $\pa(i)\subseteq S$ and $|\pa(i)|\le d\le q$. Hence $\pa(i)$ can be extended to a set $T\subseteq S$ with $|T|=q_S$, and Theorem~\ref{thm:parent-completion} gives $Q_i^{(q)}(S)=-I_i^\eps$. Conversely, if $Q_i^{(q)}(S)=-I_i^\eps$, some admissible $T\subseteq S$ satisfies $\Gamma_i(T)=0$, so Theorem~\ref{thm:parent-completion} gives $\pa(i)\subseteq T\subseteq S$. Therefore
\[ i\in\src(S) \quad\Longleftrightarrow\quad Q_i^{(q)}(S)=-I_i^\eps. \]
If a nonsource remains and $g_q(S)>0$, the definition of the ordering margin gives
\[ \min_{i\notin S\cup\src(S)}Q_i^{(q)}(S) > \min_{s\in\src(S)}Q_s^{(q)}(S), \]
so every minimizer over $i\notin S$ is a current source.

Now fix $i\in\src(S)$. The source characterization gives $Q_i^{(q)}(S)=-I_i^\eps$, while Theorem~\ref{thm:parent-completion} gives, for every $T\subseteq S$ with $|T|\le q$,
\[ J_i(T)=Q_i^{(q)}(S) \quad\Longleftrightarrow\quad \Gamma_i(T)=0 \quad\Longleftrightarrow\quad \pa(i)\subseteq T. \]
If $T=\pa(i)$, then $J_i(T)=Q_i^{(q)}(S)$, while every proper subset $U\subsetneq T$ omits a parent and satisfies $J_i(U)>Q_i^{(q)}(S)$. Conversely, suppose $J_i(T)=Q_i^{(q)}(S)$ and every proper subset $U\subsetneq T$ has $J_i(U)>Q_i^{(q)}(S)$. The displayed equivalence gives $\pa(i)\subseteq T$. If the inclusion were strict, $U=\pa(i)$ would be a proper subset with $J_i(U)=Q_i^{(q)}(S)$, a contradiction. Hence $T=\pa(i)$, which also proves the two uniqueness claims.

Begin the recursion with $S=\varnothing$, which is ancestral. At any stage, if a nonsource remains, the positive ordering margin forces every minimizer to be a current source. If no nonsource remains, every unselected node is already a current source, so any minimizer is valid. Selecting and adding a current source preserves ancestrality, and the recursion therefore produces a topological ordering. The parent characterization then recovers each selected node's true parent set from its predecessors. Thus the procedure recovers the true DAG, and every queried marginal contains at most one target node and $q$ conditioning variables.
\end{proof}

Theorem~\ref{thm:population-recovery} reduces population recovery to positivity of the ordering margin, but checking this margin over every ancestral set may be difficult. A transparent sufficient condition can instead be stated directly in terms of the noise Fisher information along causal paths. The following corollary shows that pathwise nonincrease, including equal noise Fisher information as a special case, guarantees a positive ordering margin.

\begin{corollary}[A sufficient ordering condition]
\label{cor:equal-fisher-recovery}
Under the conditions of Theorem~\ref{thm:parent-completion}, suppose $q\ge d$ and the noise Fisher information does not increase along any directed path. Then $g_q(S)>0$ at every ancestral set with a remaining nonsource, and population FiCS recovers the true DAG. Equal noise Fisher information is a special case.
\end{corollary}

\begin{proof}
Fix an ancestral set $S$. If no nonsource remains, every unselected node is a current source, and there is nothing to prove. Otherwise, fix a remaining nonsource $i$. Following incoming edges from $i$ through unselected parents must terminate at a node $s\in\src(S)$, and $s$ is an ancestor of $i$. The pathwise nonincrease assumption gives $I_s^\eps\ge I_i^\eps$. Every admissible subset $T\subseteq S$ for node $i$ omits at least one parent, so Theorem~\ref{thm:parent-completion} gives $\Gamma_i(T)>0$. Because the admissible subset collection is finite,
\[ Q_i^{(q)}(S)>-I_i^\eps\ge-I_s^\eps=Q_s^{(q)}(S)\ge\min_{t\in\src(S)}Q_t^{(q)}(S). \]
This holds for every remaining nonsource $i$. Taking the minimum over them yields $g_q(S)>0$. Theorem~\ref{thm:population-recovery} therefore gives population DAG recovery.
\end{proof}

\subsection{Proof of Proposition~\ref{prop:locality-benefit}}
\label{app:subset-population}

\begin{proof}[Proof of Proposition~\ref{prop:locality-benefit}]
Let $T\subseteq U\subseteq V\setminus\{i\}$ and define $s_{i,A}:=\partial_{x_i}\log p(X_i\mid X_A)$ for $A\in\{T,U\}$. Marginalizing the conditional distribution given $X_U$ over $X_{U\setminus T}$ and differentiating with respect to $x_i$ gives
\[ s_{i,T}=\E[s_{i,U}\mid X_i,X_T]. \]
The conditional variance decomposition and $J_i(A)=-\E[s_{i,A}^2]$ therefore yield
\[ J_i(T)-J_i(U) = \E[s_{i,U}^2]-\E[s_{i,T}^2] = \E[(s_{i,U}-s_{i,T})^2] \ge0. \]
Every subset of $S$ of size $\min\{q_1,|S|\}$ can be extended to a subset of size $\min\{q_2,|S|\}$. Applying the preceding monotonicity identity gives
\[ Q_i^{(q_1)}(S)\ge Q_i^{(q_2)}(S) \]
for every $i\notin S$. If $s\in\src(S)$, both subset sizes admit a set containing $\pa(s)$ because $q_1\ge d$. Theorem~\ref{thm:parent-completion} therefore gives
\[ Q_s^{(q_1)}(S)=Q_s^{(q_2)}(S)=-I_s^\eps. \]
Subtracting the common minimum source score from the minimum nonsource score gives $g_{q_1}(S)\ge g_{q_2}(S)$ whenever a nonsource remains.

\medskip
\noindent\emph{Strictness.}
The nonsource score $Q_i^{(q)}(S)$ can strictly decrease as $q$ increases. Consider
\[ X_1=\eps_1,\qquad X_2=\eps_2,\qquad X_3=\eps_3,\qquad X_4=3X_1+3X_2+\eps_4,\qquad X_5=\tfrac12X_4+2X_3+\eps_5, \]
where the errors are independent centered Gaussians with variances $(1,1,1,2,1/2)$. At the ancestral set $S=\{1,2,3\}$, node $4$ is the only current source and node $5$ is a nonsource. The corresponding population scores are
\begin{center}
\begin{tabular}{lccc}
\toprule
Subset size & $Q_4^{(q)}(S)$ & $Q_5^{(q)}(S)$ & $g_q(S)$\\
\midrule
$q=2$ & $-1/2$ & $-4/13$ & $5/26$\\
$q=3$ & $-1/2$ & $-1$ & $-1/2$\\
\bottomrule
\end{tabular}
\end{center}
Conditioning node $5$ on its best two roots leaves conditional variance $13/4$, whereas conditioning on all three roots leaves conditional variance one. Hence $g_2(S)>g_3(S)$, and increasing the subset size can reverse the population ordering decision.
\end{proof}

\subsection{Proof of Proposition~\ref{prop:uniform-nongaussian}}
\label{app:uniform-separation-proof}

For a scalar random variable $Y$ with differentiable density, write $I(Y):=\E[(\partial_y\log p_Y(Y))^2]$ for its location Fisher information. For each even $p$, consider the chain ANM $\mathcal M_p$ given by
\[ X_1=\eps_1,\qquad X_j=a_jX_{j-1}+\eps_j,\quad 2\le j\le p, \]
where the errors are mutually independent, $\eps_j$ follows the standard Student $t_3$ distribution for odd $j$, and $\eps_j\sim N(0,1)$ for even $j$. Choose the coefficients to satisfy
\[ \frac{13}{40}\le a_j^2\le\frac{33}{100}\quad(j\text{ even}),\qquad \frac14\le |a_j|\le\frac13\quad(j\ge3\text{ odd}). \]
The standard Student $t_3$ distribution has variance $3$ and location Fisher information $2/3$, whereas the standard Gaussian distribution has variance and location Fisher information one.

\begin{proof}
We verify lower bounds on the ordering and parent-omission margins and then the failures of the competing ordering rules. First, for every nonroot node $j$ and every parent-omitting set $T\subseteq\nd(j)$ satisfying $|T|\le1$, we claim that
\[ \Gamma_j(T)\ge\frac{2}{75}. \]
\medskip
\noindent\emph{Lower bound on the parent-omission margin.}
Since $\pa(j)=\{j-1\}$, every parent-omitting set $T\subseteq\nd(j)$ contains only variables among $X_1,\ldots,X_{j-2}$. For $j\ge3$, write
\[ X_j=a_ja_{j-1}X_{j-2}+V_j,\qquad V_j:=a_j\eps_{j-1}+\eps_j. \]
The random variable $V_j$ is independent of $(X_{j-2},X_T)$. Applying the conditional-score projection identity from the proof of Theorem~\ref{thm:parent-completion} to the structural shift $a_ja_{j-1}X_{j-2}$ and noise $V_j$ gives
\[ I_j(T)\le I(V_j). \]
For $j=2$, a parent-omitting set is empty and $X_2=V_2$, so the same inequality holds.

Suppose first that $j$ is odd. Then $\eps_{j-1}\sim N(0,1)$ and $I(\eps_j)=2/3$. Stam's inequality~\citep{stam1959information} gives
\[ \frac{1}{I(V_j)}\ge\frac{1}{I(a_j\eps_{j-1})}+\frac{1}{I(\eps_j)}=a_j^2+\frac32. \]
Since $|a_j|\ge1/4$,
\[ I(V_j)\le\left(a_j^2+\frac32\right)^{-1}\le\left(\frac1{16}+\frac32\right)^{-1}=\frac{16}{25}. \]
Consequently,
\[ \Gamma_j(T)=I_j^\eps-I_j(T)\ge\frac23-\frac{16}{25}=\frac{2}{75}. \]

Now suppose that $j$ is even. Write $\eps_{j-1}=\sqrt{3/U}\,G$, where $U\sim\chi_3^2$ and $G\sim N(0,1)$ are independent. Conditional on $U$, the variable $V_j$ is Gaussian with variance $1+c/U$, where $c:=3a_j^2\ge39/40$. The marginal score of $V_j$ is the conditional expectation of its Gaussian score given $V_j$. Conditional Jensen's inequality therefore yields
\[ I(V_j)\le\E\!\left[\frac{1}{1+c/U}\right]=\E\!\left[\frac{U}{U+c}\right]\le\E\!\left[\frac{U}{U+39/40}\right]. \]
We next verify a convenient rational upper bound for the last expectation. Let
\[ R(u):=25000-8000u+1050u^2-56u^3+u^4. \]
Cauchy--Schwarz gives
\[ \E\!\left[\frac{1}{U+39/40}\right]\ge\frac{\{\E[R(U)]\}^2}{\E[(U+39/40)R(U)^2]}. \]
Using $\E[U^k]=\prod_{\ell=0}^{k}(2\ell+1)$ for $U\sim\chi_3^2$ gives
\[ \E[R(U)]=11815,\qquad \E[(U+39/40)R(U)^2]=\frac{3252350883}{8}. \]
It follows that
\[ \E\!\left[\frac{U}{U+39/40}\right]=1-\frac{39}{40}\E\!\left[\frac{1}{U+39/40}\right]\le\frac{721171976}{1084116961}<\frac23-\frac1{1000}. \]
Therefore,
\[ \Gamma_j(T)=1-I_j(T)\ge1-I(V_j)>\frac13+\frac1{1000}>\frac{2}{75}. \]
The two parity cases establish $\Delta_{\rm par}=2/75$ as a lower bound on the parent-omission margin.

\medskip
\noindent\emph{Lower bound on the ordering margin.}
At every ancestral set $S$ for which a nonsource remains, we have
\[ g_1(S)\ge\frac1{1000}. \]
An ancestral set in a directed chain consists of an initial segment of the chain. Its unique current source $s$ has $Q_s^{(1)}(S)=-I_s^\eps$. For every remaining nonsource $j$, every admissible subset omits its parent, so the lower bounds on the parent-omission margin above give
\[ Q_j^{(1)}(S)-Q_s^{(1)}(S)=I_s^\eps-I_j^\eps+\min_{T\subseteq S,\;|T|=\min\{1,|S|\}}\Gamma_j(T). \]
If $j$ is odd, the right-hand side is at least $2/75$ because $I_s^\eps\ge2/3=I_j^\eps$. If $j$ is even, it is greater than $2/3-1+(1/3+1/1000)=1/1000$. Taking the minimum over nonsources establishes $\Delta_{\rm ord}=1/1000$.

\medskip
\noindent\emph{Failure of the competing ordering rules.}
For a nonleaf $j$, independence and the zero means of the noise scores give
\[ \E[(\partial_j\log p(X))^2] =I_j^\eps+a_{j+1}^2I_{j+1}^\eps. \]
Every odd nonleaf has joint Fisher information at most $2/3+33/100=299/300<1$, whereas the unique leaf $p$ has joint Fisher information one. Thus the first CaPS minimizer is a nonleaf. For LearnDAG, the unique initial source has variance $\Var(X_1)=3$, whereas
\[ \Var(X_2)=3a_2^2+1\le199/100<3. \]
Hence $X_1$ cannot be a first-step minimizer; because it is the unique source, the first selected node is necessarily a nonsource.
The two margin lower bounds are independent of $p$, so the family is uniformly Fisher-separated at $q=1$. Theorem~\ref{thm:population-recovery} gives population DAG recovery, and the preceding calculations give the two claimed first-step failures.
\end{proof}

\subsection{Proof of Theorem~\ref{thm:finite}}
\label{app:finite-recovery}

\begin{proof}[Proof of Theorem~\ref{thm:finite}]
On $\mathcal E(\epsilon)$, the uniform error
\[ U:=\max_{(i,T)\in\mathcal Q_{p,q}}|\widehat J_i(T)-J_i(T)| \]
satisfies $U\le\epsilon$. For an ancestral set $S$, define
\[ \widehat Q_i^{(q)}(S):= \min_{\substack{T\subseteq S,\,|T|=q_S}}\widehat J_i(T). \]
The uniform local score error bound implies
\[ |\widehat Q_i^{(q)}(S)-Q_i^{(q)}(S)|\le\epsilon \]
for every $i\notin S$. Hence the empirical difference between the best nonsource score and the best current-source score is at least $\Delta_{\rm ord}-2\epsilon>0$. The ordering rule therefore selects a current source. Adding it to $S$ preserves ancestrality, so Stage~1 recursively returns a topological ordering.

Fix a node $i$ and let $S$ be its predecessors in this order, so $\pa(i)\subseteq S\subseteq\nd(i)$. Every parent-containing set has population score $-I_i^\eps$. Thus, for any strict parent-containing superset $T\supsetneq\pa(i)$,
\[ \{\widehat J_i(T)+\lambda|T|\} - \{\widehat J_i(\pa(i))+\lambda|\pa(i)|\} \ge\lambda-2\epsilon>0. \]
If $T$ omits a parent, Theorem~\ref{thm:parent-completion} and the lower bound on the parent-omission margin give $J_i(T)-J_i(\pa(i))\ge\Delta_{\rm par}$. Since $|\pa(i)|\le d$, the largest possible reduction in the parent penalty is $d\lambda$, and hence
\[ \{\widehat J_i(T)+\lambda|T|\} - \{\widehat J_i(\pa(i))+\lambda|\pa(i)|\} \ge\Delta_{\rm par}-d\lambda-2\epsilon>0. \]
Therefore $\pa(i)$ is the unique penalized minimizer. The same argument applies to every node, so Stage~2 recovers all parent sets and Algorithm~\ref{alg:fics} returns the true DAG.
\end{proof}

For the covariance-estimation result in Appendix~\ref{app:gaussian-specialization}, write the recovery tolerance of Theorem~\ref{thm:finite} as
\begin{equation}
\epsilon_{\rm alg}:=\frac12\min\{\Delta_{\rm ord},\lambda,\Delta_{\rm par}-d\lambda\}.
\label{eq:estar}
\end{equation}
Uniform local score error strictly below $\epsilon_{\rm alg}>0$ ensures DAG recovery by Algorithm~\ref{alg:fics}.

\subsection{Detailed Operator Formulation of Assumption~\ref{ass:stein-main}}
\label{app:stein-assumptions}

We use the notation of Section~\ref{sec:guarantees} and introduce the operators needed for the proofs below.

Let $\phi_W(x):=k_W(x,\cdot)$, and let $S_W:\mathcal H_W\to L^2(P_W)$ be the canonical inclusion operator. Define
\[ L_W:=S_W^*S_W = \E\{\phi_W(X_W)\otimes\phi_W(X_W)\}, \]
where $\otimes$ denotes the rank-one operator, $\|\cdot\|_{\mathrm{HS}}$ denotes the Hilbert-Schmidt norm, and $\|\cdot\|_{\mathrm{op}}$ denotes the induced operator norm. The integral operator $\mathsf T_W$ in Equation~\eqref{eq:l2-approximation} admits the factorization $\mathsf T_W=S_WS_W^*$ because, for every $h\in L^2(P_W)$,
\[ (S_WS_W^*h)(x) = \E\{k_W(x,X_W)h(X_W)\} = (\mathsf T_Wh)(x). \]
Thus, the integral and factorized representations give the same operator $\mathsf T_W$ and the same value of $A(\rho)$. Notice that $L_W$ and $\mathsf T_W$ are distinct: the former acts on $\mathcal H_W$, whereas the latter acts on $L^2(P_W)$.

Define
\[ \zeta^s_{W,i}:=\E\{\partial_i k_W(X_W,\cdot)\}, \qquad \zeta^r_{W,i}:=\E\{\partial_i^2 k_W(X_W,\cdot)\}. \]
The first- and second-order Stein identities in Assumption~\ref{ass:stein-main} take the operator form
\begin{equation} S_W^*s_{W,i}=-\zeta^s_{W,i}, \qquad S_W^*r_{W,i}=\zeta^r_{W,i}. \label{eq:stein-operator-equations} \end{equation}

All regularity, target-norm, and approximation bounds are those of Assumption~\ref{ass:stein-main}, uniformly over the local query class.

\subsection{Proof of Theorem~\ref{thm:stein-consistency} and Corollary~\ref{cor:stein-dag-consistency}}
\label{app:stein-consistency}

Write $\rho:=\eta/n$, let $\sigma_s:=-1$ and $\sigma_r:=1$, and define
\[ \begin{aligned} \mu_W&:=\E\{\phi_W(X_W)\}, & \widehat\mu_W&:=\frac1n\sum_{a=1}^n\phi_W(X_W^{(a)}),\\ \widehat L_W&:=\frac1n\sum_{a=1}^n\phi_W(X_W^{(a)})\otimes\phi_W(X_W^{(a)}), & \widehat\zeta^s_{W,i}&:=\frac1n\sum_{a=1}^n\partial_i k_W(X_W^{(a)},\cdot),\\ && \widehat\zeta^r_{W,i}&:=\frac1n\sum_{a=1}^n\partial_i^2 k_W(X_W^{(a)},\cdot). \end{aligned} \]
For $u\in\{s,r\}$, define the corresponding operator estimator by
\[ \widehat u_{W,i}:= \sigma_u(\widehat L_W+\rho I)^{-1}\widehat\zeta^u_{W,i}. \]
The representer form of the empirical normal equations gives
\[ (K_W+\eta I_n)\widehat{\mathbf s}_{W,i}=-\mathbf D_{W,i}^{(1)}, \qquad (K_W+\eta I_n)\widehat{\mathbf r}_{W,i}=\mathbf D_{W,i}^{(2)}. \]
Hence the sample-value vectors of these operator estimators coincide with the matrix estimators in Equation~\eqref{eq:fics-estimator}.

\begin{lemma}[Uniform concentration of the empirical Stein quantities]
\label{lem:uniform-forcing}
Under Assumption~\ref{ass:stein-main}, with probability at least $1-\delta$,
\[ \max_{(i,T)\in\mathcal Q_{p,q}} \left\{ \|\widehat L_W-L_W\|_{\mathrm{HS}}, \|\widehat\mu_W-\mu_W\|_{\mathcal H_W}, \|\widehat\zeta^s_{W,i}-\zeta^s_{W,i}\|_{\mathcal H_W}, \|\widehat\zeta^r_{W,i}-\zeta^r_{W,i}\|_{\mathcal H_W} \right\} \le b_{\delta}. \]
\end{lemma}

\begin{proof}
Fix a local query $(i,T)$ and abbreviate $\phi:=\phi_W(X_W)$. Define the centered Hilbert-valued variables
\[ Y^L:=\phi\otimes\phi-L_W,\qquad Y^\mu:=\phi-\mu_W, \]
\[ Y^s:=\partial_i k_W(X_W,\cdot)-\zeta^s_{W,i},\qquad Y^r:=\partial_i^2 k_W(X_W,\cdot)-\zeta^r_{W,i}. \]
The first variable takes values in the Hilbert space of Hilbert--Schmidt operators, and the remaining three take values in $\mathcal H_W$. All four variables are mean zero. Assumption~\ref{ass:stein-main}, in the detailed notation of Appendix~\ref{app:stein-assumptions}, gives
\[ \|Y^L\|_{\mathrm{HS}}\le2\kappa_0^2,\qquad \|Y^\mu\|_{\mathcal H_W}\le2\kappa_0,\qquad \|Y^s\|_{\mathcal H_W}\le2\kappa_1,\qquad \|Y^r\|_{\mathcal H_W}\le2\kappa_2 \]
almost surely, uniformly over all local queries. Indeed, $\|\phi\otimes\phi\|_{\mathrm{HS}}=\|\phi\|_{\mathcal H_W}^2$, while the norm of each expectation is bounded by the corresponding expected norm.

The Bernstein inequality for sums of independent, mean-zero Hilbert-space-valued random variables~\citep{pinelis1994optimum} therefore implies that, for each of the four empirical averages and every $x>0$,
\[ \Pp\!\left(\left\|\frac1n\sum_{a=1}^nY_a\right\|>C_0\left\{\sqrt{\frac{x}{n}}+\frac{x}{n}\right\}\right)\le2e^{-x}, \]
where $C_0$ depends only on $\kappa_0,\kappa_1,$ and $\kappa_2$. Apply this inequality with $x=\ell_\delta:=\log(16N_{p,q}/\delta)$. A union bound over the four quantities and at most $N_{p,q}$ local queries gives total failure probability at most $8N_{p,q}e^{-\ell_\delta}\le\delta$. Enlarging $C_B$ if necessary makes the common right-hand side equal to $b_\delta=C_B\{\sqrt{\ell_\delta/n}+\ell_\delta/n\}$, proving the lemma.
\end{proof}

\begin{proof}[Proof of Theorem~\ref{thm:stein-consistency}]
Work on the event in Lemma~\ref{lem:uniform-forcing}, and suppose that $b_{\delta}\le\rho\le1$.

\medskip
\noindent\emph{Population ridge approximation and empirical target error.}
Fix a local query and $u\in\{s,r\}$, write $g:=u_{W,i}$, and define the population-regularized target
\[ u_\rho:=(L_W+\rho I)^{-1}S_W^*g = \sigma_u(L_W+\rho I)^{-1}\zeta^u_{W,i}. \]
The identities $L_W=S_W^*S_W$ and $\mathsf T_W=S_WS_W^*$ give
\[ S_Wu_\rho = \mathsf T_W(\mathsf T_W+\rho I)^{-1}g, \qquad \|g-S_Wu_\rho\|_{L^2(P_W)} \le A(\rho). \]
Spectral calculus and the uniform target bound in Assumption~\ref{ass:stein-main} also give
\[ \|u_\rho\|_{\mathcal H_W} \le \frac{M}{2\sqrt\rho}, \]
because the singular values of $(L_W+\rho I)^{-1}S_W^*$ are $\sqrt t/(t+\rho)$ and $\sup_{t\ge0}\sqrt t/(t+\rho)=1/(2\sqrt\rho)$. Subtracting the empirical and population normal equations yields
\[ \widehat u_{W,i}-u_\rho = (\widehat L_W+\rho I)^{-1} \left[ \sigma_u(\widehat\zeta^u_{W,i}-\zeta^u_{W,i}) + (L_W-\widehat L_W)u_\rho \right]. \]
Put $b:=b_{\delta}$. Since $\|\widehat L_W-L_W\|_{\mathrm{op}}\le b\le\rho$, we have $L_W\preceq\widehat L_W+\rho I$ and hence
\[ \|S_W(\widehat L_W+\rho I)^{-1}\|_{\mathrm{op}} \le\rho^{-1/2}. \]
Indeed, with $A_W:=\widehat L_W+\rho I$, the operator inequality $L_W\preceq A_W$ gives $A_W^{-1}L_WA_W^{-1}\preceq A_W^{-1}\preceq\rho^{-1}I$, and the squared norm on the preceding display equals $\|A_W^{-1}L_WA_W^{-1}\|_{\mathrm{op}}$.
The concentration lemma and the bound on $\|u_\rho\|_{\mathcal H_W}$ therefore imply
\[ \|S_W(\widehat u_{W,i}-u_\rho)\|_{L^2(P_W)} \le \rho^{-1/2} \left(b+\frac{bM}{2\sqrt\rho}\right) \le C\frac b\rho. \]
Combining the last display with the population approximation bound gives, uniformly over all local queries and $u\in\{s,r\}$,
\[ \|S_W\widehat u_{W,i}-u_{W,i}\|_{L^2(P_W)} \le A(\rho)+C\frac{b_{\delta}}{\rho}. \]
Using $\|A_W^{-1}\|_{\mathrm{op}}\le1/\rho$ in the same inverse identity gives
\[ \|\widehat u_{W,i}-u_\rho\|_{\mathcal H_W}\le\frac b\rho+\frac{bM}{2\rho^{3/2}}\le1+\frac{M}{2\sqrt\rho}. \]
Together with the bound on $\|u_\rho\|_{\mathcal H_W}$ and $\rho\le1$, this proves
\[ \|\widehat u_{W,i}\|_{\mathcal H_W}\le\frac{C}{\sqrt\rho}. \]

\medskip
\noindent\emph{Propagation to the local Fisher scores.}
Fix a local query and abbreviate $s:=s_{W,i}$, $r:=r_{W,i}$, $\widehat s:=\widehat s_{W,i}$, and $\widehat r:=\widehat r_{W,i}$. The empirical and population scores satisfy
\[ \widehat J_i(T) = \langle\widehat\mu_W,\widehat r\rangle_{\mathcal H_W} - \langle\widehat s,\widehat L_W\widehat s\rangle_{\mathcal H_W}, \qquad J_i(T) = \E\{r_{W,i}(X_W)\}-\|s_{W,i}\|_{L^2(P_W)}^2. \]
Let $e_u:=\|S_W\widehat u-u\|_{L^2(P_W)}$ for $u\in\{s,r\}$. Adding and subtracting the corresponding population expectations gives
\[ |\widehat J_i(T)-J_i(T)| \le b_{\delta}\|\widehat r\|_{\mathcal H_W} + e_r + b_{\delta}\|\widehat s\|_{\mathcal H_W}^2 + e_s(2M+e_s). \]
Here the linear term follows from $|\langle\widehat\mu_W-\mu_W,\widehat r\rangle|\le b_\delta\|\widehat r\|$ and $|\E\{S_W\widehat r-r\}|\le e_r$. For the quadratic term,
\[ \left|\langle\widehat s,\widehat L_W\widehat s\rangle-\|s\|_{L^2(P_W)}^2\right|\le b_\delta\|\widehat s\|_{\mathcal H_W}^2+\left|\|S_W\widehat s\|_{L^2(P_W)}^2-\|s\|_{L^2(P_W)}^2\right|, \]
and the last absolute value is at most $e_s(2M+e_s)$ because $\|S_W\widehat s\|_{L^2(P_W)}\le M+e_s$.
The target bounds above imply $e_s,e_r\le Ce_{\delta}(\rho)$ and $\|\widehat s\|_{\mathcal H_W},\|\widehat r\|_{\mathcal H_W}\le C/\sqrt{\rho}$. Since $\rho\le1$, the empirical-evaluation terms are bounded by $Cb_{\delta}/\rho$, which is absorbed into $Ce_{\delta}(\rho)$ because $b_{\delta}/\rho\le e_{\delta}(\rho)$. Taking the maximum over all local queries therefore gives
\[
\max_{(i,T)\in\mathcal Q_{p,q}}
|\widehat J_i(T)-J_i(T)|
\le
C\left\{
e_{\delta}(\rho)+e_{\delta}(\rho)^2
\right\}.
\]
This is Equation~\eqref{eq:stein-uniform-rate}. Lemma~\ref{lem:uniform-forcing} holds with probability at least $1-\delta$, completing the proof.
\end{proof}

\begin{proof}[Proof of Corollary~\ref{cor:stein-dag-consistency}]
Assumption~\ref{ass:high-dimensional} gives $L=o(n)$ and hence $b\to0$. Fix any $\delta\in(0,1)$. Since $\log N_{p,q}=O(L)$ and $\log(1/\delta)$ is constant, $b_{\delta}\le C_\delta b$ for all sufficiently large $n$. Moreover, $b/\rho\to0$ implies $b_\delta\le\rho$ eventually, so Theorem~\ref{thm:stein-consistency} is applicable at this fixed confidence level. The definitions of $e$ and $R$ therefore give a constant $M_\delta<\infty$ such that $\Pp(U>M_\delta R)\le\delta$ eventually, where
\[ U:= \max_{(i,T)\in\mathcal Q_{p,q}} |\widehat J_i(T)-J_i(T)| = O_{\Pp}(R). \]
Because $\delta$ is arbitrary, this is precisely $U=O_{\Pp}(R)$. Since $R=o(m)$, Slutsky's theorem gives $U/m=(U/R)(R/m)=o_{\Pp}(1)$ and therefore $\Pp(2U<m)\to1$. On the event $2U<m$, Theorem~\ref{thm:finite} applies with $\epsilon=U$, so Algorithm~\ref{alg:fics} returns the true DAG. Hence $\Pp(\widehat G=G)\to1$.
\end{proof}

\section{RKHS Source Analysis for Regularized Kernel Stein Estimation}
\label{app:rkhs-source-result}

The main analysis requires only uniform $L^2$ Stein regularity and approximation, allowing the first- and second-order Stein targets to lie outside their kernel RKHSs. This appendix replaces the abstract approximation condition with the stronger requirement that both targets belong uniformly to source spaces of the RKHS covariance operators. In return, it yields convergence in RKHS norm and an explicit regularization-bias rate. These stronger conclusions are not needed for the main DAG recovery result.

\subsection{Uniform RKHS source regularity and strengthened rates}

For $u\in\{s,r\}$, let $\beta_u\in(0,1]$ denote its source exponent and define $\beta_*:=\min\{\beta_s,\beta_r\}$. For a positive self-adjoint operator $L_W$, let $L_W^{\beta_u}$ denote its spectral power.

\begin{assumption}[Uniform local RKHS source regularity]
\label{ass:stein-rkhs}
Retain the uniform kernel regularity, target-norm bounds, and Stein identities of Assumption~\ref{ass:stein-main}. For each $u\in\{s,r\}$, suppose that the $L^2(P_W)$ target $u_{W,i}$ admits an RKHS representative, denoted by the same symbol, and that there exist a fixed exponent $\beta_u\in(0,1]$ and an element $f^u_{W,i}\in\mathcal H_W$ such that
\[ u_{W,i}=L_W^{\beta_u}f^u_{W,i},\qquad \sup_{n\ge1}\max_{(i,T)\in\mathcal Q_{p,q}}\max_{u\in\{s,r\}}\|f^u_{W,i}\|_{\mathcal H_W}\le F \]
for some $F<\infty$.
\end{assumption}

The source condition implies the uniform approximation requirement $A(\rho)\to0$ in Assumption~\ref{ass:stein-main}. Indeed, the intertwining identity $(\mathsf T_W+\rho I)^{-1}S_W=S_W(L_W+\rho I)^{-1}$ implies $A(\rho)\le C\rho^{\beta_*}$, so the original approximation condition follows automatically. Define
\[
\omega_\delta(\rho)
:=
\rho^{\beta_*}+\frac{b_\delta}{\rho}.
\]

\begin{theorem}[Uniform RKHS error bounds and explicit source rate]
\label{thm:stein-rkhs}
Under Assumption~\ref{ass:stein-rkhs}, set the matrix ridge in Equation~\eqref{eq:fics-estimator} to $\eta_W=n\rho$ for every queried $W$. For every $\delta\in(0,1)$ and $b_{\delta}\le\rho\le1$, with probability at least $1-\delta$,
\[ \max_{(i,T)\in\mathcal Q_{p,q}}\|\widehat u_{W,i}-u_{W,i}\|_{\mathcal H_W}\le C\left\{\rho^{\beta_u}+\frac{b_{\delta}}{\rho}\right\},\qquad u\in\{s,r\}, \]
and

\begin{equation}
\max_{(i,T)\in\mathcal Q_{p,q}}
\left|\widehat J_i(T)-J_i(T)\right|
\le
C\left\{
\omega_{\delta}(\rho)+\omega_{\delta}(\rho)^2
\right\}.
\label{eq:rkhs-score-rate}
\end{equation}
Here $C$ is independent of $n,p,q,\delta$, and $\rho$.
\end{theorem}

When $b_{\delta}\to0$, balancing $\rho^{\beta_*}$ and $b_{\delta}/\rho$ gives $\rho\asymp b_{\delta}^{1/(1+\beta_*)}$ and the leading rate $b_{\delta}^{\beta_*/(1+\beta_*)}$. Thus, Theorem~\ref{thm:stein-rkhs} sharpens Theorem~\ref{thm:stein-consistency} when the abstract approximation condition is replaced by the stronger RKHS source condition. In the asymptotic notation of Corollary~\ref{cor:stein-dag-consistency}, put $R_{\rm src}:=b^{\beta_*/(1+\beta_*)}$. If $d\ge1$ and $dR_{\rm src}\to0$, the balanced ridge parameter $\rho\asymp b^{1/(1+\beta_*)}$ and, for example, $\lambda:=\sqrt{R_{\rm src}/d}$ satisfy $R_{\rm src}=o(\lambda)$ and $d\lambda\to0$. 
If $d=0$, taking $\lambda:=\sqrt{R_{\rm src}}$ gives $R_{\rm src}=o(\lambda)$ and $\lambda\to0$, so redundant nonempty parent sets are excluded.
Hence the source condition supplies compatible ridge-parameter and parent penalty sequences whenever $dR_{\rm src}\to0$.

\subsection{RKHS stability under the source condition}
\label{app:rkhs-stability-result}

For $u\in\{s,r\}$, define
\[ \widehat u_{W,i}:=\sigma_u(\widehat L_W+\rho I)^{-1}\widehat\zeta^u_{W,i},\qquad u_{W,i,\rho}:=(L_W+\rho I)^{-1}L_Wu_{W,i}, \]
where $\sigma_s:=-1$ and $\sigma_r:=1$.

\begin{lemma}[RKHS stability under a source condition]
\label{lem:rkhs-source-stability}
On the event of Lemma~\ref{lem:uniform-forcing}, Assumption~\ref{ass:stein-rkhs} implies, uniformly over $(i,T)\in\mathcal Q_{p,q}$ and $u\in\{s,r\}$,
\[ \|u_{W,i,\rho}-u_{W,i}\|_{\mathcal H_W}\le C\rho^{\beta_u},\qquad \|\widehat u_{W,i}-u_{W,i,\rho}\|_{\mathcal H_W}\le C\frac{b_{\delta}}{\rho}. \]
\end{lemma}

This lemma isolates the only nontrivial operator-stability step needed for Theorem~\ref{thm:stein-rkhs}. Appendix~\ref{app:auxiliary-lemma-proof} gives its proof and integrates the remaining score-propagation argument directly into the proof of the theorem.

\subsection{Examples inside and outside the RKHS source class}
\label{app:rkhs-source-examples}

The following propositions consider nonlinear Cauchy and linear Gaussian families, respectively.

\begin{proposition}[A connected nonlinear family satisfying the RKHS source condition]
\label{prop:connected-admissible}
Fix $r_0\in\{1,2\}$ and set $q=r_0$. For any $p>r_0$, let the first $r_0$ variables be independent standard Cauchy roots and let
\begin{equation} X_k=\eps_k\quad(k\le r_0),\qquad X_i=\sum_{k=1}^{r_0} b_{ki}\tanh(X_k)+\eps_i\quad(i>r_0),\qquad b_{ki}\in\{-1,1\}, \label{eq:connected-cauchy} \end{equation}
where all errors are independent standard Cauchy variables. Arbitrary node-label permutations are allowed. This connected nonlinear family is uniformly Fisher-separated at $q=r_0$, with positive lower bounds on the ordering and parent-omission margins independent of $p$ and the coefficient signs. Moreover, there is a deterministic finite-dimensional kernel family satisfying Assumption~\ref{ass:stein-rkhs} uniformly, and hence also Assumption~\ref{ass:stein-main}, with $\beta_s=\beta_r=1$.
\end{proposition}

\begin{proposition}[An $L^2$-admissible family outside the Gaussian-RBF source class]
\label{prop:gaussian-rbf-example}
Fix $r_0\in\{1,2\}$ and set $q=r_0$. Let $X_1,\ldots,X_{r_0}$ be independent standard Gaussian roots and, for every $i>r_0$, let
\[ X_i=\sum_{k=1}^{r_0} b_{ki}X_k+\eps_i,\qquad b_{ki}\in\{-1,1\},\qquad \eps_i\overset{\mathrm{ind}}\sim N(0,1). \]
Use a deterministic Gaussian RBF kernel with fixed positive bandwidth in each local dimension. This family satisfies Assumption~\ref{ass:stein-main} uniformly and is uniformly Fisher-separated at $q=r_0$. However, it does not satisfy Assumption~\ref{ass:stein-rkhs} for this fixed-bandwidth Gaussian RBF kernel family.
\end{proposition}

The constructed kernel in Proposition~\ref{prop:connected-admissible} is not the kernel used in the experiments. Proofs of both propositions are given in Appendix~\ref{app:auxiliary-lemma-proof}.

\section{Linear Gaussian Specialization}
\label{app:gaussian-specialization}

This appendix specializes the estimator-agnostic recovery result in Theorem~\ref{thm:finite} to a Gaussian covariance plug-in estimator; it does not repeat the regularized kernel Stein analysis of Theorem~\ref{thm:stein-consistency}. For a Gaussian marginal on $W=\{i\}\cup T$, with covariance matrix $\Sigma_W:=\operatorname{Cov}(X_W)$,
\[ J_i(T)=-\frac{1}{\Var(X_i\mid X_T)}=-(\Sigma_W^{-1})_{ii}. \]
Because $v\mapsto-1/v$ is increasing, FiCS then agrees with the bounded-subset conditional-variance criterion of \citet[Eq.~(5)]{chen2019equal}. Here $\|B\|_{\mathrm{op}}:=\sup_{\|v\|_2=1}\|Bv\|_2$ is the induced matrix operator norm, consistent with the operator-norm notation used above, and $\|B\|_{\mathrm F}$ is the Frobenius norm.

\begin{corollary}[Finite-sample Gaussian DAG recovery]
\label{cor:gaussian}
Consider a linear Gaussian ANM satisfying the conditions of Theorem~\ref{thm:parent-completion}, with $\max\{1,d\}\le q\le p-1$ and positive lower bounds $\Delta_{\rm ord},\Delta_{\rm par}$ on the ordering and parent-omission margins. Use the parent selection rule with $\lambda>0$ and $d\lambda<\Delta_{\rm par}$, and let $\epsilon_{\rm alg}>0$ be the recovery tolerance in Equation~\eqref{eq:estar}. For every queried $W$, define
\[ \bar X_W:=\frac1n\sum_{a=1}^nX_W^{(a)},\qquad \widehat\Sigma_W:=\frac1n\sum_{a=1}^n(X_W^{(a)}-\bar X_W)(X_W^{(a)}-\bar X_W)^\top,\qquad \widehat J_i^{\rm G}(T):=-(\widehat\Sigma_W^{-1})_{ii}. \]
Here the superscript distinguishes the Gaussian covariance plug-in estimator from the regularized kernel Stein estimator in Equation~\eqref{eq:fics-estimator}; both estimate the same population score $J_i(T)$. Suppose that every $\Sigma_W$ has eigenvalues in $[c_{\min},c_{\max}]$, where $0<c_{\min}\le c_{\max}<\infty$. Then there is a universal constant $C$ such that, for every $\delta\in(0,1)$, FiCS using the scores $\widehat J_i^{\rm G}(T)$ recovers the true DAG with probability at least $1-\delta$ whenever
\begin{equation} n\ge C\left[q\left\{1+\log\!\left(\frac{p}{q}\right)\right\}+\log\!\left(\frac p\delta\right)\right]\left(\frac{c_{\max}}{c_{\min}}\right)^2\left(1+\frac{1}{c_{\min}^2\epsilon_{\rm alg}^2}\right). \label{eq:gaussian-rate} \end{equation}
\end{corollary}

Equation~\eqref{eq:gaussian-rate} separates the combinatorial term from the condition-number and recovery-tolerance factors. The tolerance $\epsilon_{\rm alg}$ also depends on the indegree through $d\lambda$; the displayed combinatorial term alone therefore does not establish a minimax rate for growing $d$. Appendix~\ref{app:gaussian-proof} proves the bound.

\section{Proofs of Auxiliary Results}
\label{app:auxiliary-lemma-proof}

This appendix proves the strengthened RKHS results and examples in Appendix~\ref{app:rkhs-source-result}, followed by the Gaussian corollary in Appendix~\ref{app:gaussian-specialization}.

\subsection{Proof of Lemma~\ref{lem:rkhs-source-stability}}

\begin{proof}
Fix a local query, suppress $(W,i)$, and write $u=L^{\beta_u}f^u$. Spectral calculus gives
\[ u_\rho-u=-\rho(L+\rho I)^{-1}L^{\beta_u}f^u. \]
Since $\sup_{t\ge0}\rho t^{\beta_u}/(t+\rho)\le C\rho^{\beta_u}$ for $\beta_u\in(0,1]$, the uniform bound on $\|f^u\|_{\mathcal H}$ proves the first claim. It also gives $\|u_\rho\|_{\mathcal H}\le\|u\|_{\mathcal H}\le C$ because $\|L\|_{\mathrm{op}}\le\kappa_0^2$.

The Stein identity gives $Lu=\sigma_u\zeta^u$. Subtracting the empirical and population regularized equations yields
\[ \widehat u-u_\rho=(\widehat L+\rho I)^{-1}\left[\sigma_u(\widehat\zeta^u-\zeta^u)+(L-\widehat L)u_\rho\right]. \]
On the event of Lemma~\ref{lem:uniform-forcing}, both empirical deviations are at most $b_{\delta}$, while $\|(\widehat L+\rho I)^{-1}\|_{\mathrm{op}}\le1/\rho$. The uniform bound on $\|u_\rho\|_{\mathcal H}$ therefore proves the second claim.
\end{proof}

\subsection{Proof of Theorem~\ref{thm:stein-rkhs}}
\label{app:rkhs-source-proof}

\begin{proof}
Lemma~\ref{lem:uniform-forcing} provides a common event of probability at least $1-\delta$. On this event, Lemma~\ref{lem:rkhs-source-stability} and the triangle inequality give
\[ d_u:=\max_{(i,T)\in\mathcal Q_{p,q}}\|\widehat u_{W,i}-u_{W,i}\|_{\mathcal H_W}\le C\left\{\rho^{\beta_u}+\frac{b_{\delta}}{\rho}\right\},\qquad u\in\{s,r\}. \]

It remains to propagate these target errors to the local Fisher scores. Fix a local query and abbreviate the corresponding targets by $s,r,\widehat s,\widehat r$. The empirical and population scores satisfy
\[ \widehat J_i(T)=\langle\widehat\mu_W,\widehat r\rangle_{\mathcal H_W}-\langle\widehat s,\widehat L_W\widehat s\rangle_{\mathcal H_W},\qquad J_i(T)=\langle\mu_W,r\rangle_{\mathcal H_W}-\langle s,L_Ws\rangle_{\mathcal H_W}. \]
The first difference is bounded by
\[ b_{\delta}\|\widehat r\|_{\mathcal H_W}+\|\mu_W\|_{\mathcal H_W}\|\widehat r-r\|_{\mathcal H_W}, \]
and $\|\mu_W\|_{\mathcal H_W}\le\kappa_0$. For the quadratic term, adding and subtracting $\langle\widehat s,L_W\widehat s\rangle_{\mathcal H_W}$ gives
\[ b_{\delta}\|\widehat s\|_{\mathcal H_W}^2+\|L_W\|_{\mathrm{op}}\|\widehat s-s\|_{\mathcal H_W}\left(\|\widehat s\|_{\mathcal H_W}+\|s\|_{\mathcal H_W}\right). \]
Assumption~\ref{ass:stein-rkhs} and $\|L_W\|_{\mathrm{op}}\le\kappa_0^2$ give $\|u_{W,i}\|_{\mathcal H_W}\le\kappa_0^{2\beta_u}F$, so both $\|s\|_{\mathcal H_W}$ and $\|r\|_{\mathcal H_W}$ are uniformly bounded. Moreover, $\|\widehat u\|_{\mathcal H_W}\le\|u\|_{\mathcal H_W}+d_u$. Hence
\[ \max_{(i,T)\in\mathcal Q_{p,q}}|\widehat J_i(T)-J_i(T)|\le C\left\{d_s+d_r+d_s^2+b_{\delta}(1+d_s+d_r)^2\right\}. \]

Since $b_{\delta}\le\rho\le1$, we have $\omega_{\delta}(\rho)\le2$, $d_s+d_r\le C\omega_{\delta}(\rho)$, and $b_{\delta}\le\omega_{\delta}(\rho)$. Substitution into the preceding display bounds its right-hand side by $C\{\omega_{\delta}(\rho)+\omega_{\delta}(\rho)^2\}$, proving Equation~\eqref{eq:rkhs-score-rate}.
\end{proof}

\subsection{Proof of Proposition~\ref{prop:connected-admissible}}

\begin{proof}
All errors have Fisher information $1/2$. We first verify uniform Fisher separation. Let $i>r_0$ be a leaf and let $T\subseteq\nd(i)$, $|T|\le r_0$, omit at least one parent. The joint density of the roots and any finite collection of leaves is everywhere positive because it is a product of positive Cauchy error densities under the triangular structural transformation. Consequently, conditional on $X_T$, every omitted root has support $\mathbb R$ almost surely. Holding the other roots fixed and varying an omitted root changes $M_i=\sum_{k=1}^{r_0}b_{ki}\tanh(X_k)$ because $b_{ki}\neq0$ and $\tanh$ is strictly monotone. Hence $M_i$ is not $\sigma(X_T)$-measurable. The equality characterization in Theorem~\ref{thm:parent-completion} gives $\Gamma_i(T)>0$.

For fixed $r_0$, a local query contains at most $r_0+1$ variables. Up to coordinate permutations, its law is determined by finitely many root--leaf roles and coefficient-sign patterns, independently of $p$. The resulting parent-omission margins range over a finite collection of positive numbers, whose minimum gives a common lower bound $\Delta_{\rm par}>0$ on the parent-omission margin. At any ancestral set with a remaining nonsource, every current source is an unselected root and every nonsource is a leaf missing at least one root parent. Since the noise Fisher information is $1/2$ at every node, Theorem~\ref{thm:population-recovery} then gives $g_q(S)\ge\Delta_{\rm par}$. Thus the family is uniformly Fisher-separated.

We next construct an RKHS family and verify all parts of Assumption~\ref{ass:stein-rkhs}. For each local dimension $k\le r_0+1$, let $\mathcal D_k$ be the finite collection of marginal densities arising from the local types above. The standard Cauchy density is positive and $C^\infty$, and each fixed-order derivative divided by the density is bounded. All derivatives of $\tanh$ are also bounded. Repeated differentiation of the joint density therefore produces the joint density multiplied by a bounded function. Dominated differentiation under the integral over unobserved roots shows that every $f\in\mathcal D_k$ is positive and $C^\infty$ and that the first- and second-order targets $\partial_j\log f$ and $(\partial_j^2f)/f$, together with all derivatives needed below, are bounded. The same domination shows that the marginal density and its relevant derivatives vanish at infinity, so both Stein integration-by-parts identities hold.

Form a finite feature vector $\psi_k$ from the constant function and all distinct target functions $\partial_j\log f$ and $(\partial_j^2f)/f$ with $f\in\mathcal D_k$ and $j\in[k]$, removing pointwise linear redundancies, and set $k_k(x,y):=\langle\psi_k(x),\psi_k(y)\rangle$. For every queried $W$, use $k_W:=k_{|W|}$ and $\mathcal H_W:=\operatorname{span}(\psi_{|W|})$. This kernel family is deterministic, positive semidefinite, and $C^4$. Bounded derivatives of the finite feature vectors give constants $\kappa_0,\kappa_1,\kappa_2$ that are common to all local types. Every local target belongs to $\mathcal H_W$ and has a uniformly bounded $L^2(P_W)$ norm.

For any local law $P_W$ and nonzero $h\in\mathcal H_W$, removal of pointwise redundancies makes $h$ a nonzero continuous function. Since $P_W$ has an everywhere positive density, $\langle h,L_Wh\rangle_{\mathcal H_W}=\E\{h(X_W)^2\}>0$. Thus $L_W$ is positive definite and invertible on $\mathcal H_W$. There are only finitely many local types, so $\sup_W\|L_W^{-1}\|_{\mathrm{op}}<\infty$. Taking $f^u_{W,i}:=L_W^{-1}u_{W,i}$ gives $u_{W,i}=L_Wf^u_{W,i}$ with a common bound on $\|f^u_{W,i}\|_{\mathcal H_W}$. Hence Assumption~\ref{ass:stein-rkhs} holds with $\beta_s=\beta_r=1$.
\end{proof}

\subsection{Proof of Proposition~\ref{prop:gaussian-rbf-example}}

\begin{proof}
Every queried marginal has dimension at most $r_0+1$. Up to coordinate permutation, its covariance is determined by finitely many root--leaf roles and coefficient-sign patterns. Each such covariance is positive definite because every observed leaf contains an independent unit-variance error. Finiteness of the type collection therefore gives common positive lower and finite upper eigenvalue bounds. Consequently, the Gaussian first-order targets are linear, the second-order targets are quadratic, and their $L^2(P_W)$ norms are uniformly bounded. Gaussian tails justify the first- and second-order Stein identities. For a fixed positive bandwidth in each of the finitely many local dimensions, the Gaussian RBF kernel is deterministic and $C^\infty$, and its feature and first two derivative-feature norms are uniformly bounded by constants $\kappa_0,\kappa_1,$ and $\kappa_2$.

It remains to verify the uniform approximation condition. Fix a local law $P_W$, denote its RBF bandwidth by $h_W>0$, and let $v\in L^2(P_W)$. The finite signed measure $\nu(dx):=v(x)P_W(dx)$ has finite total variation. For the translation-invariant Gaussian kernel, Fourier inversion gives
\[ \langle v,\mathsf T_Wv\rangle_{L^2(P_W)}=c_W\int_{\mathbb R^{|W|}}|\widehat\nu(\omega)|^2\exp\{-h_W^2\|\omega\|^2/2\}\,d\omega, \]
where $c_W>0$ depends only on the bandwidth and dimension. The Gaussian spectral density is strictly positive, so $\mathsf T_Wv=0$ implies $\widehat\nu=0$ and hence $v=0$ in $L^2(P_W)$. 
Thus $\ker(\mathsf T_W)=\{0\}$ and, because $\mathsf T_W$ is self-adjoint, $\overline{\operatorname{Ran}(\mathsf T_W)}=L^2(P_W)$.
Spectral calculus now gives $\|\rho(\mathsf T_W+\rho I)^{-1}u_{W,i}\|_{L^2(P_W)}\to0$ for both targets. Only finitely many local distribution types occur, so taking their maximum preserves convergence and makes all constants uniform. This verifies Assumption~\ref{ass:stein-main}.

The stronger RKHS source condition fails for this kernel family. The Gaussian RBF kernel has bounded diagonal, so every $f$ in its RKHS is bounded by $|f(x)|\le\|f\|_{\mathcal H_W}\sqrt{k_W(x,x)}$. Already for a singleton standard Gaussian root, the targets are $s(x)=-x$ and $r(x)=x^2-1$, both unbounded. They therefore do not belong to this Gaussian RBF RKHS and cannot equal $L_W^{\beta_u}f^u_{W,i}$ for any $\beta_u>0$ and $f^u_{W,i}\in\mathcal H_W$. Thus Assumption~\ref{ass:stein-rkhs} fails specifically for the stated fixed-bandwidth Gaussian RBF kernel family.

Finally, we verify uniform Fisher separation. Suppose a leaf omits $r_{\rm miss}\ge1$ root parents. Conditional on any included roots, the remaining roots form a standard Gaussian vector $R\in\mathbb R^{r_{\rm miss}}$. Because $|T|\le r_0$, the other variables in $T$ provide at most $r_{\rm miss}$ noisy linear measurements $BR+\xi$, where every row of $B$ has entries in $\{-1,1\}$ and $\xi$ has identity covariance. The conditional covariance of $R$ is $(I+B^\top B)^{-1}$. If $a\in\{-1,1\}^{r_{\rm miss}}$ is the target leaf's loading vector, then
\[ \Var(a^\top R\mid BR+\xi)=a^\top(I+B^\top B)^{-1}a\ge\frac{\|a\|_2^2}{1+\|B\|_{\mathrm{op}}^2}\ge\frac{r_{\rm miss}}{1+r_{\rm miss}^2}, \]
because $B$ has at most $r_{\rm miss}$ rows and $\|B\|_{\mathrm{op}}^2\le\|B\|_{\mathrm F}^2\le r_{\rm miss}^2$. The conditional variance of the target leaf is therefore at least $1+r_{\rm miss}/(1+r_{\rm miss}^2)$. Since its noise Fisher information is one, the parent-omission margin is at least $1/3$ for $r_0=1$ and $2/7$ for $r_0=2$. Equal noise Fisher information across nodes then supplies the same positive lower bound on the ordering margin.
\end{proof}

\subsection{Proof of Corollary~\ref{cor:gaussian}}
\label{app:gaussian-proof}

\begin{proof}
Let $k:=|W|\le q+1$. Standard Gaussian concentration for the centered sample covariance implies that, for every $x>0$, with probability at least $1-2e^{-x}$,
\begin{equation} \|\widehat\Sigma_W-\Sigma_W\|_{\mathrm{op}}\le C_0c_{\max}\left\{\sqrt{\frac{k+x}{n}}+\frac{k+x}{n}\right\}. \label{eq:wishart-local} \end{equation}
Apply Equation~\eqref{eq:wishart-local} with $x=\log(2N_{p,q}/\delta)$ and take a union bound over all local queries. With $\ell_\delta:=\log(16N_{p,q}/\delta)$, there is a universal constant $C_2$ such that, with probability at least $1-\delta$,
\begin{equation} \max_{(i,T)\in\mathcal Q_{p,q}}\|\widehat\Sigma_W-\Sigma_W\|_{\mathrm{op}}\le C_2c_{\max}\left\{\sqrt{\frac{\ell_{\delta}}{n}}+\frac{\ell_{\delta}}{n}\right\}. \label{eq:wishart-uniform} \end{equation}

Since $\ell_{\delta}\le C_3[q\{1+\log(p/q)\}+\log(p/\delta)]$, the sample-size condition in Equation~\eqref{eq:gaussian-rate} implies $n\gtrsim\ell_{\delta}(c_{\max}/c_{\min})^2$. After increasing the universal constant if necessary, the right-hand side of Equation~\eqref{eq:wishart-uniform} is at most $c_{\min}/2$. Weyl's inequality then gives $\lambda_{\min}(\widehat\Sigma_W)\ge c_{\min}/2$ simultaneously over all local queries, so every $\widehat\Sigma_W$ is invertible. The inverse perturbation identity yields
\[ \|\widehat\Sigma_W^{-1}-\Sigma_W^{-1}\|_{\mathrm{op}}\le\|\widehat\Sigma_W^{-1}\|_{\mathrm{op}}\|\widehat\Sigma_W-\Sigma_W\|_{\mathrm{op}}\|\Sigma_W^{-1}\|_{\mathrm{op}}\le2c_{\min}^{-2}\|\widehat\Sigma_W-\Sigma_W\|_{\mathrm{op}}. \]
Consequently,
\[ \max_{(i,T)\in\mathcal Q_{p,q}}|\widehat J_i^{\rm G}(T)-J_i(T)|\le2c_{\min}^{-2}\max_{(i,T)\in\mathcal Q_{p,q}}\|\widehat\Sigma_W-\Sigma_W\|_{\mathrm{op}}. \]
The remaining factor in Equation~\eqref{eq:gaussian-rate} makes the last display at most $\epsilon_{\rm alg}/2$ after a further universal enlargement of $C$. Equation~\eqref{eq:estar} and Theorem~\ref{thm:finite} therefore give DAG recovery with probability at least $1-\delta$.
\end{proof}

\FloatBarrier
\section{Implementation Details}\label{app:exp-details}

\subsection{FiCS Implementation}\label{sec:implementation}
\paragraph{Score estimation and numerical solution.} For $W$ with $k$ variables and observations $z_a=X_W^{(a)}$, Equation~\eqref{eq:fics-estimator} uses
\[
\begin{gathered}
k_W(z,z')=h_W^{-1}\exp\{-\|z-z'\|^2/(2h_W^2)\},\qquad (K_W)_{ab}=k_W(z_a,z_b),\\
\eta_W=\max\left\{10^{-10},\frac{r\operatorname{tr}(K_W)}{nk}\right\}.
\end{gathered}
\]
Here $h_W$ is $u>0$ times the median of all $n^2$ pairwise Euclidean distances, including the diagonal, and $r>0$ is the ridge multiplier. One Cholesky factorization serves both derivative systems and all target coordinates of $W$. A nonpositive bandwidth or failed factorization raises an error. The regularized kernel Stein estimator uses all selected observations; synthetic and real-data preprocessing are specified in Appendices~\ref{app:synthetic-generation} and~\ref{app:sachs}, respectively. Section~5 does not directly analyze this empirical bandwidth-and-ridge rule, but gives the local-score accuracy required for exact recovery.

\paragraph{Search and tie breaking.} Ordering uses subsets of size $\min(q,|S|)$; parent selection uses all predecessor subsets of size at most $q$. Parent-objective ties prefer fewer parents, followed by lexicographic order. Parent selection uses parent penalty $\lambda=c\sqrt{\log(p)/n}$, with the parent penalty multiplier $c$ specified in Appendix~\ref{app:configurations}.

\paragraph{Incremental score updates.} Once $|S|\ge q$, retain each remaining candidate's minimum score. After adding node $j$, evaluate only new subsets $U\cup\{j\}$ with $U\subseteq S$ and $|U|=q-1$. Together with the old subsets, these exhaust the $q$-subsets of $S\cup\{j\}$, so updating the minimum preserves the original criterion and deterministic tie breaking. Before $|S|=q$, use the entire selected set. For $A_{p,q}=\sum_{t=0}^{q-1}(p-t)$ and $1\le q<p$, this reduces ordering-score lookups from $A_{p,q}+\binom{p+1}{q+2}$ to $A_{p,q}+\binom{p}{q+1}$, or from $O(p^{q+2})$ to $O(p^{q+1})$ for fixed $q$. It changes neither the distinct kernel factorizations nor parent selection.

\subsection{Baseline Implementations}
\label{app:baseline-implementations}

\paragraph{CaPS and SCORE.}
CaPS~\citep{xu2024caps} and SCORE~\citep{rolland2022score} use their implemented graph-estimation procedures. For the real-data benchmarks, CaPS uses the authors' \texttt{train\_test} function with native R-CAM pruning, and SCORE uses the authors' implementation with CAM pruning. Hyperparameter settings are given in Appendix~\ref{app:configurations}.

\paragraph{NoGAM.}
For the synthetic experiments, NoGAM~\citep{montagna2023nogam} uses five-fold kernel-ridge regression followed by CAM pruning. Its adapter passes the declared RBF parameter $\gamma$ to the regressor; the public implementation mistakenly passes ridge $\alpha$ in that position. The correction is applied to both tuning and test runs. For the real-data benchmarks, NoGAM uses the public DoDiscover \texttt{NoGAM()} implementation with native PyGAM pruning. Numerical routines in these real-data runs are unchanged; only data and path interfaces are adapted. The corresponding settings, including the effective KRR RBF coefficient, are given in Appendix~\ref{app:configurations}.

\paragraph{DAGMA.}
Nonlinear DAGMA uses an MLP with architecture $[p,10,1]$, weight decay $\lambda_2=0.005$, and double precision. It starts with $T=4$ optimization stages; cyclic outputs trigger fresh fits at $T=6,7,8$, retaining the first acyclic output. The same nonlinear implementation and retry rule are used for the synthetic and real-data benchmarks. %

\subsection{Hyperparameter Selection and Configurations}
\label{app:configurations}

\paragraph{Synthetic experiments.}
Tuning uses ten independent DAGs at each of $p=20,50$ and $n=500,1000$, giving 40 inputs disjoint from the test data. Each method has four predetermined configurations in Table~\ref{tab:exp-tuning}. Candidates are ranked by successful-fit fraction, F1 on inputs completed by all tied candidates, SHD$/p$, and finally candidate order. The selected configuration is fixed across all synthetic test dimensions, sample sizes, and robustness conditions. Reference edges are used only to evaluate tuning candidates.

\begin{table}[!htbp]
\centering
\captionsetup{skip=2pt}
\renewcommand{\arraystretch}{0.96}
\setlength{\tabcolsep}{4pt}
\setlength{\aboverulesep}{1.4pt}
\setlength{\belowrulesep}{1.4pt}
\caption{Hyperparameter candidates evaluated on the independent synthetic tuning datasets.}
\label{tab:exp-tuning}
\footnotesize
\begin{tabular}{@{}lllll@{}}
\toprule
Method & Candidate 1 & Candidate 2 & Candidate 3 & Candidate 4\\
\midrule
FiCS $(u,r,c)$ & $(1.5,.5,.2)$ & $(.75,.5,.2)$ & $(1.5,.1,.2)$ & $(1.5,.5,.8)$\\
CaPS $(\lambda_1,\lambda_2)$ & $(50,50)$ & $(2,50)$ & $(2,25)$ & $(5,50)$\\
SCORE $(\eta,\xi)$ & $(.001,.001)$ & $(.001,.05)$ & $(.01,.001)$ & $(.01,.05)$\\
NoGAM $(\alpha,\gamma)$ & $(.01,.1)$ & $(.001,.1)$ & $(.1,.1)$ & $(.01,1)$\\
DAGMA $(\lambda_1,t)$ & $(.02,.3)$ & $(.02,.1)$ & $(.005,.3)$ & $(.005,.1)$\\
\bottomrule
\end{tabular}
\end{table}

The selected configurations are FiCS $(u,r,c)=(1.5,0.5,0.8)$, CaPS $(\lambda_1,\lambda_2)=(50,50)$, SCORE $(\eta,\xi)=(0.001,0.001)$, NoGAM $(\alpha,\gamma)=(0.001,0.1)$, and DAGMA $(\lambda_1,t)=(0.02,0.3)$. Here, $\xi$ is the pruning level and $t$ is the DAGMA edge threshold. CaPS uses $\lambda_1$ for initial parent-score pruning and $\lambda_2$ for subsequent edge addition; SCORE sets $\eta_G=\eta_H=\eta$; and NoGAM fixes its Stein regularization and pruning level at $0.001$.

\paragraph{Real-data experiments.}
FiCS uses bandwidth multiplier $u=0.5$, ridge multiplier $r=1.0$, parent penalty multiplier $c=0.0003$, and $q=\min\{2,p-1\}$ in both stages. The kernel, bandwidth, and matrix ridge are defined in Appendix~\ref{sec:implementation}; parent selection minimizes $\widehat J_i(T)+\lambda|T|$ with $\lambda=c\sqrt{\log(p)/n}$. The rule gives $q=2$ on all three datasets; each search is also restricted to the available predecessors. Appendix~\ref{app:realdata-sensitivity} reports sensitivity to this common configuration on the same benchmark inputs.

CaPS uses $\lambda_1=\lambda_2=50$, pre-pruning and edge addition enabled, Stein regularization $0.001$, and native R-CAM pruning at $0.001$. SCORE uses $\eta_G=\eta_H=0.001$, CAM cutoff $0.001$, variance normalization disabled, and variance as the dispersion statistic. NoGAM uses the public DoDiscover \texttt{NoGAM()} defaults: five folds, KRR regularization $0.01$, $\eta_G=\eta_H=0.001$, and native PyGAM pruning at $\alpha=0.05$ with ten cubic splines. Preliminary neighborhood selection is disabled. The effective KRR RBF coefficient in this implementation is $0.01$. DAGMA uses the synthetic configuration $(\lambda_1,t)=(0.02,0.3)$ and the architecture and retry rule in Appendix~\ref{app:baseline-implementations}.

Each method uses the same configuration and execution rule across the three real-data benchmarks.

\FloatBarrier
\section{Experimental Design}
\label{app:experimental-design}
\label{app:exp-protocol}

\subsection{Synthetic Data Generation}
\label{app:synthetic-generation}
We generate ER3 DAGs by sampling a random topological order and including each admissible forward edge independently with probability $3/(p-1)$. Thus the expected mean total degree is three before imposing the indegree cap of two. When necessary, two parents are retained using independent random priorities. Each nonroot mechanism is linear with probability $0.5$, with a random coefficient sign and magnitude sampled uniformly from $[0.5,1]$, and otherwise is sampled from a zero-mean Gaussian process on the full parent vector with kernel $k(u,v)=\exp(-\|u-v\|^2/2)$. The structural errors are independent $\mathcal N(0,1)$. 

\subsection{Evaluation Metrics}
\label{app:exp-metrics}
For reference edges $E$ and estimated ordering $\widehat\pi$, $A_{\mathrm{top}}=|\{(i,j)\in E:\widehat\pi(i)<\widehat\pi(j)\}|/|E|$ is the fraction of reference edges ordered forward. F1 compares directed edges; a reversal counts as one false positive and one false negative. SHD counts a reversal as one edit. Means and sample standard deviations are computed over valid DAG estimates; a missing graph is not assigned F1 zero. The instances included in these means can differ between methods. 

\subsection{Main Synthetic Benchmark}
\label{app:main-synthetic-benchmark}

Figure~\ref{fig:exp-main-synthetic} evaluates 30 independently generated DAGs at each $p\in\{50,100,200\}$, using the generator in Appendix~\ref{app:synthetic-generation} and nested sample sizes $n\in\{50,100,200,500,1000\}$. We compare FiCS with CaPS, SCORE, NoGAM, and nonlinear DAGMA on the same datasets within each setting. FiCS uses subset size $q=2$ in both ordering estimation and parent selection. Each method retains the configuration selected on independent tuning DAGs in Appendix~\ref{app:configurations} across all test dimensions and sample sizes. Figure~\ref{fig:exp-main-synthetic} reports mean directed-edge F1 and one sample standard deviation using the aggregation rule in Appendix~\ref{app:exp-metrics}.

\subsection{Effect of the Conditioning Size}
\label{app:ordering-subset-limit}

Figure~\ref{fig:exp-conditioning} uses the same 30 DAGs and observations per dimension, with $p\in\{50,100,150,200\}$ and $n=100$. FiCS varies only its conditioning size for ordering over $q=1,2,3$ and the entire selected set; the parent-selection subset size remains two and the parent penalty is $0.8\sqrt{\log(p)/n}$, while the regularized kernel Stein estimator and its tuning remain fixed. CaPS provides an external reference with its native pruning and edge supplementation. For CaPS, $A_{\mathrm{top}}$ uses the order before edge supplementation, whereas F1 evaluates its final DAG.

\subsection{CaPS Numerical Failures}
\label{app:exp-computation}

CaPS returns no DAG when its native GAM pruning encounters a nonfinite covariance matrix and the eigendecomposition fails. Table~\ref{tab:caps-numerical-failures} lists all 79 such numerical failures in Figure~\ref{fig:exp-main-synthetic}.

\begin{table}[!htbp]
\centering\small
\setlength{\tabcolsep}{10pt}
\caption{CaPS numerical failures among 30 attempts per setting in Figure~\ref{fig:exp-main-synthetic}. Only settings with a recorded numerical error are listed.}
\label{tab:caps-numerical-failures}
\begin{tabular}{rrr}
\toprule
$p$ & $n$ & Numerical failures\\
\midrule
50 & 50 & 1\\
100 & 50 & 22\\
100 & 100 & 5\\
200 & 50 & 28\\
200 & 100 & 19\\
200 & 200 & 4\\
\bottomrule
\end{tabular}
\end{table}

In Figure~\ref{fig:exp-conditioning}, the same numerical error occurs in 0, 5, 14, and 19 of 30 CaPS runs at $p=50,100,150,200$, respectively.

\subsection{Runtime Measurements}
\label{app:exp-runtime}

For each setting, runtime is measured by refitting three of the 30 graph--sample inputs used in Figure~\ref{fig:exp-main-synthetic}, with the same inputs and fixed configurations for all methods. Reported wall-clock time covers fitting, including graph construction and pruning, on the same server; data generation and calibration are excluded.

\begin{table}[!htbp]
\centering\small
\setlength{\tabcolsep}{4pt}
\caption{Wall-clock fitting time (minutes) on three paired graph--sample inputs per setting. Entries are medians of three completed fits, except CaPS at $(p,n)=(200,200)$. Its marked value is the time of the single completed fit (1/3), not a median.}
\label{tab:synthetic-runtime}
\begin{tabular*}{\linewidth}{@{\extracolsep{\fill}}rrrrrrr@{}}
\toprule
$p$ & $n$ & FiCS & CaPS & SCORE & NoGAM & DAGMA\\
\midrule
50 & 200 & 0.45 & 0.37 & 0.30 & 0.64 & 5.61\\
50 & 500 & 2.84 & 1.90 & 0.51 & 1.39 & 13.48\\
100 & 200 & 3.65 & 16.72 & 1.15 & 2.72 & 14.52\\
100 & 500 & 23.18 & 8.77 & 1.91 & 5.55 & 42.07\\
200 & 200 & 58.75 & 1025.06$^{\dagger}$ & 11.85 & 23.27 & 173.34\\
200 & 500 & 307.32 & 630.60 & 20.36 & 50.09 & 239.59\\
\bottomrule
\end{tabular*}
\end{table}

Table~\ref{tab:synthetic-runtime} provides an empirical check of the complexity analysis in Section~\ref{sec:fics-computation}. With $q=2$, the number of distinct marginal factorizations is bounded by $\sum_{k=1}^{3}\binom pk=(p^3+5p)/6$, which increases approximately eightfold when $p$ doubles. Correspondingly, increasing $p$ from $50$ to $100$ increases FiCS runtime by factors of $8.1$ and $8.2$ at $n=200$ and $n=500$, respectively, closely matching this prediction. The increase from $p=100$ to $200$ is steeper, at $13.3$--$16.1$ times, so the agreement is strongest over the smaller dimensions. Increasing $n$ from $200$ to $500$ raises runtime by $5.2$--$6.4$ times, below the $15.6$-fold increase associated with cubic factorization cost; thus, the asymptotic $O(n^3)$ cost does not manifest as cubic wall-clock scaling over this sample-size range.

\subsection{Real-Data Benchmarks}
\label{app:sachs}
\label{app:rivers}

All five methods in Table~\ref{tab:sachs-main} receive the same observations in their original measurement units, with no external logarithm or variance standardization. We retain each implementation's internal transformations, including FiCS's mean-centering for score estimation. F1 and SHD are computed against the directed references below; SHD counts a reversal once (Appendix~\ref{app:exp-metrics}).

\paragraph{Data and reference graphs.}
For Danube, we use the 12-station selection and 11-edge spatial reference from the official code of \citet{gnecco2021heavy}. Intersecting dates over its prescribed 32-station pool and retaining June--August before 2010 gives $N=4{,}600$ complete daily observations from 1960--2009. We take 500 equally spaced row indices, $\lfloor\operatorname{linspace}(0,N-1,500)\rfloor$, preserving chronological order. For Sachs, we use all 853 observations of the 11-protein benchmark matrix and its 17-edge reference DAG distributed with CaPS~\citep{sachs2005,xu2024caps}.

For CausalRivers, we use the official East Germany release~\citep{stein2025causalrivers}. Its sampler yields 7,521 connected five-station induced subnetworks. Candidates are ranked by SHA-256 hashes of ordered station identifiers with the prefix \texttt{FiCS-rivers-20260925:}. We apply the official six-hour rounding and mean aggregation, followed by linear interpolation, and take 500 equally spaced jointly finite rows per subnetwork. Eligibility requires at least 500 jointly finite rows and a nonconstant value in every coordinate. We use the first 30 eligible candidates: one constant-input candidate is excluded for every method and replaced by the next ranked candidate. Each reference has four edges. We transpose the release's effect-row labels to the cause-row adjacency convention used throughout this paper.

\paragraph{Scope of evaluation.}
FiCS uses $q=2$ on all three datasets, with the settings in Appendix~\ref{app:configurations}. This is a computational cap, not an estimate of the true maximum indegree: the references have maximum indegree five on Danube, three on Sachs, and one to four across CausalRivers subnetworks. Thus, the cap can exclude reference parent sets.

These experiments assess agreement with biological or spatial reference graphs; they do not verify the theorem's i.i.d. ANM or recovery assumptions. River flows exhibit temporal dependence and may share rainfall drivers or unobserved upstream stations. The 30 CausalRivers subnetworks overlap in stations and observation times. Their reported SD describes variation across subnetworks, not uncertainty from 30 independent datasets.

\FloatBarrier

\section{Additional Experiments}
\label{app:additional-experiments}

\subsection{Sensitivity to Graph Structure, Mechanisms, and Noise}
\label{app:exp-graph-results}

These tests examine whether the main random-DAG results depend on a particular graph family, nonlinear mechanism, or error distribution. Figure~\ref{fig:exp-robustness} changes one factor at a time at $p=50,n=200$, using 20 DAGs per condition and the fixed synthetic configurations. The scale-free (SF) graph family begins with an undirected preferential-attachment graph with two new edges per arriving node, orients its edges according to a random order, and caps the indegree at two. The additive alternative replaces nonlinear GP mechanisms by $\sum_{i\in\pa(j)}b_{ij}\tanh(X_i)$ while retaining the $0.5$ probability of a linear mechanism. Student-$t_5$ errors are scaled by $\sqrt{3/5}$ to have unit variance, and unequal Gaussian scales use noise standard deviation $\exp(U_j)$ with $U_j\sim\operatorname{Uniform}[-\log 2,\log 2]$. These DAGs are not screened for the sufficient recovery conditions.

\begin{figure}[H]
\centering
\includegraphics[width=0.92\linewidth]{method_legend.pdf}\par
\begin{subfigure}[t]{0.325\linewidth}
\centering\includegraphics[width=\linewidth]{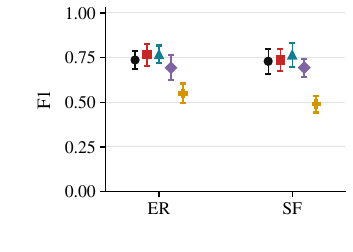}
\caption{Graph family}
\end{subfigure}\hfill
\begin{subfigure}[t]{0.325\linewidth}
\centering\includegraphics[width=\linewidth]{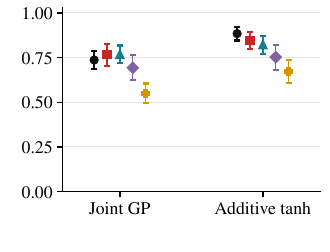}
\caption{Nonlinear mechanism}
\end{subfigure}\hfill
\begin{subfigure}[t]{0.325\linewidth}
\centering\includegraphics[width=\linewidth]{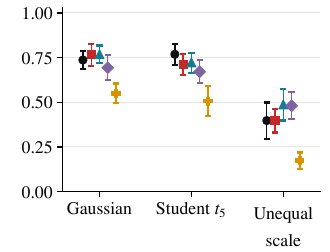}
\caption{Error distribution}
\end{subfigure}
\caption{\textbf{Sensitivity to the data-generating model.} Mean F1 and one-SD error bars over 20 DAGs per condition, at $p=50,n=200$. Each panel changes one factor from the ER3, joint-GP, unit-Gaussian reference condition. Horizontal offsets separate methods within each category.}
\label{fig:exp-robustness}
\end{figure}

FiCS's mean F1 is similar on ER3 and SF graphs ($0.736$ and $0.729$), and increases to $0.884$ for additive nonlinear mechanisms. The unequal-noise-scale condition reduces it to $0.397$, compared with $0.768$ for unit-variance Student-$t_5$ errors. Thus, the method tolerates these graph and mechanism changes better than heterogeneity in Gaussian noise scales; the results do not establish robustness to arbitrary error distributions.

\subsection{Effect of the Linear-Mechanism Proportion}
\label{app:linear-fraction}

The main generator mixes linear and nonlinear mechanisms equally. To determine how this choice affects the comparison, we vary the probability of a linear nonroot mechanism over $\alpha\in\{0,0.25,0.5,0.75,1\}$. The endpoints are all joint-parent GP mechanisms and all linear mechanisms. We retain $p=50,n=200$, maximum indegree two, unit Gaussian noise, and the fixed configurations of all five main methods, with $q=2$ for FiCS. Graphs, coefficient draws, node-type uniforms, and primitive noise are paired across $\alpha$; GP realizations are drawn separately across mechanism conditions. All 20 replicates per method and condition return DAGs (500 fits in total).

\begin{figure}[H]
\centering
\begin{subfigure}[c]{0.375\linewidth}
\centering\includegraphics[width=\linewidth]{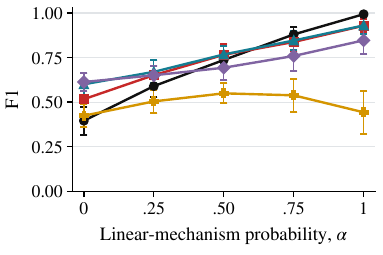}
\caption{F1}
\end{subfigure}\hfill
\begin{subfigure}[c]{0.375\linewidth}
\centering\includegraphics[width=\linewidth]{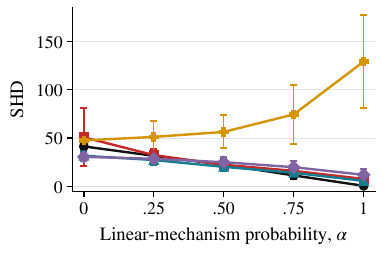}
\caption{SHD}
\end{subfigure}\hfill
\begin{minipage}[c]{0.22\linewidth}
\centering\includegraphics[width=\linewidth]{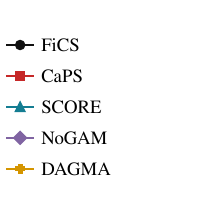}
\end{minipage}
\caption{\textbf{Effect of the linear-mechanism proportion.} Mean and one-SD error bars over all 20 DAGs per condition, at $p=50,n=200$. Increasing $\alpha$ replaces nonlinear GP mechanisms with linear ones; $\alpha=0.5$ matches the main generator. All methods retain their fixed synthetic settings.}
\label{fig:linear-fraction}
\end{figure}

FiCS F1 increases from $0.396$ at $\alpha=0$ to $0.993$ at $\alpha=1$, while SHD decreases from $41.4$ to $0.6$. NoGAM and SCORE perform better at $\alpha=0$ (F1 $0.613$ and $0.600$); FiCS leads at $\alpha=0.75$ and $1$. Thus, performance depends on mechanism composition, and the main mixed-mechanism results should not be read as uniform superiority on purely nonlinear models.

\FloatBarrier

\subsection{Real-Data Sensitivity Analysis}
\label{app:realdata-sensitivity}

We vary one parameter at a time around $(u,r,c)=(0.5,1.0,0.0003)$ and $q_0=\min\{2,p-1\}$, retaining all other parameters and the same 32 inputs. Table~\ref{tab:realdata-sensitivity} varies the score-estimation parameters and subset size. 
Table~\ref{tab:realdata-parent-multiplier} varies only the parent penalty multiplier.

\paragraph{Bandwidth multiplier, ridge multiplier, and subset size.}
We evaluate $u,r\in\{0.25,0.5,1\}$ and $q\in\{q_0,\min(3,p-1),\min(5,p-1),p-1\}$ with $c=0.0003$. Changing $q$ changes both ordering estimation and parent selection; it is not an isolated parent-selection comparison. For CausalRivers, $p=5$, so $\min\{5,p-1\}$ and $p-1$ are the same setting. Changes in the bandwidth and ridge multipliers give mean CausalRivers F1 between $0.520$ and $0.563$. On Sachs, changing $u$ from $0.5$ to $0.25$ or $1.0$ reduces F1 from $0.538$ to $0.429$ or $0.417$. Larger $q$ improves recall on Danube and CausalRivers but adds false positives; F1 is nonmonotone and SHD increases relative to $q_0$.

\begin{table}[!htbp]
\centering\small
\caption{Sensitivity to bandwidth multiplier, ridge multiplier, and subset size on all three datasets. The common row uses $(u,r,c)=(0.5,1.0,0.0003)$ and $q=\min\{2,p-1\}$. Each other row changes only its named parameter. CausalRivers values are means over the same 30 subnetworks.}
\label{tab:realdata-sensitivity}
\label{tab:sachs-sensitivity}
\setlength{\tabcolsep}{4pt}
\begin{tabular*}{\linewidth}{@{\extracolsep{\fill}}llrrrrrr@{}}
\toprule
 & & \multicolumn{2}{c}{Danube} & \multicolumn{2}{c}{Sachs} & \multicolumn{2}{c}{CausalRivers}\\
\cmidrule(lr){3-4}\cmidrule(lr){5-6}\cmidrule(l){7-8}
Parameter & Value & F1$\uparrow$ & SHD$\downarrow$ & F1$\uparrow$ & SHD$\downarrow$ & F1$\uparrow$ & SHD$\downarrow$\\
\midrule
Common configuration & -- & 0.452 & 17 & 0.538 & 11 & 0.563 & 4.53\\
\midrule
Bandwidth multiplier $u$ & 0.25 & 0.467 & 16 & 0.429 & 12 & 0.548 & 4.77\\
 & 1.0 & 0.467 & 16 & 0.417 & 13 & 0.520 & 4.93\\
\midrule
Ridge multiplier $r$ & 0.25 & 0.500 & 16 & 0.444 & 14 & 0.554 & 4.63\\
 & 0.5 & 0.452 & 17 & 0.500 & 13 & 0.551 & 4.67\\
\midrule
Subset size $q$ & $\min\{3,p-1\}$ & 0.474 & 20 & 0.500 & 13 & 0.579 & 5.03\\
 & $\min\{5,p-1\}$ & 0.383 & 29 & 0.500 & 13 & 0.562 & 5.63\\
 & $p-1$ & 0.353 & 33 & 0.500 & 13 & 0.562 & 5.63\\
\bottomrule
\end{tabular*}
\end{table}

\paragraph{Parent penalty multiplier.}
We vary $c\in\{0.0001,0.0003,0.0005,0.001,0.002,0.003\}$ in $\lambda=c\sqrt{\log(p)/n}$, fixing $u=0.5$, $r=1.0$, and $q=q_0$. Changing $c$ affects parent selection but leaves the ordering and $A_{\mathrm{top}}$ unchanged.

\begin{table}[!htbp]
\centering\small
\caption{Sensitivity to the parent penalty multiplier $c$ in $\lambda=c\sqrt{\log(p)/n}$. Bandwidth and ridge multipliers remain $0.5$ and $1.0$, respectively, and $q=\min\{2,p-1\}$ in both stages. Bold $c$ identifies the main-table setting. Edges is the number of selected edges; CausalRivers entries are means over the same 30 subnetworks.}
\label{tab:realdata-parent-multiplier}
\setlength{\tabcolsep}{3pt}
\begin{tabular*}{\linewidth}{@{\extracolsep{\fill}}lrrrrrrrrr@{}}
\toprule
 & \multicolumn{3}{c}{Danube} & \multicolumn{3}{c}{Sachs} & \multicolumn{3}{c}{CausalRivers}\\
\cmidrule(lr){2-4}\cmidrule(lr){5-7}\cmidrule(l){8-10}
$c$ & F1$\uparrow$ & SHD$\downarrow$ & Edges & F1$\uparrow$ & SHD$\downarrow$ & Edges & F1$\uparrow$ & SHD$\downarrow$ & Edges\\
\midrule
0.0001 & 0.438 & 18 & 21 & 0.500 & 11 & 11 & 0.567 & 4.57 & 6.80\\
\textbf{0.0003} & 0.452 & 17 & 20 & 0.538 & 11 & 9 & 0.563 & 4.53 & 6.63\\
0.0005 & 0.452 & 17 & 20 & 0.538 & 11 & 9 & 0.534 & 4.67 & 6.37\\
0.001 & 0.308 & 18 & 15 & 0.538 & 11 & 9 & 0.514 & 4.73 & 6.17\\
0.002 & 0.308 & 18 & 15 & 0.560 & 10 & 8 & 0.510 & 4.77 & 6.13\\
0.003 & 0.240 & 19 & 14 & 0.560 & 10 & 8 & 0.512 & 4.73 & 6.10\\
\bottomrule
\end{tabular*}
\end{table}

Increasing $c$ produces sparser graphs but has different effects across datasets (Table~\ref{tab:realdata-parent-multiplier}). From $c=0.0001$ to $0.003$, Danube F1 decreases from $0.438$ to $0.240$, whereas Sachs F1 increases from $0.500$ to $0.560$; CausalRivers F1 ranges from $0.510$ to $0.567$ across the grid. The common setting $c=0.0003$ therefore does not maximize F1 on every dataset.

\FloatBarrier

\subsection{Gaussian Generator Regimes and Scope of Ordering Guarantees}
\label{app:uumc-sortability}

To examine the scope of the population ordering guarantee beyond the main generator, we use the six Gaussian coefficient and noise-scale regimes of UUMCdata~\citep{herman2025uumc}, version 0.0. Each regime contains 30 Erd\H{o}s--R\'enyi DAGs with $p=20$ and edge probability $0.1$. We do not screen graphs for the theorem's recovery conditions. FiCS uses subset size $q=5$, the maximum indegree in this cohort; this subset size makes every true parent set admissible but does not ensure a positive ordering margin.

Table~\ref{tab:uumc-population-atop} compares analytic population ordering with ordering at $n=100$. FiCS uses $q=5$ and its fixed synthetic estimator settings for the finite-sample results; VarSort and $R^2$Sort use empirical variance and linear-$R^2$ scores. Only ordering accuracy is compared. In 50--50 and iSCM, sample-size-dependent normalization changes the generated SCM, so the two columns are separate regime comparisons rather than a learning curve for a fixed SCM.

At the population level, FiCS has the highest $A_{\mathrm{top}}$ in four of the six regimes; $R^2$Sort leads in DaO and UUMC. FiCS orders all true edges correctly under unit-variance noise but does not do so in the other five regimes. Thus a subset size at least the maximum indegree does not imply the positive ordering margin required by Theorem~\ref{thm:population-recovery}.

\begin{table}[H]
\centering\small
\setlength{\tabcolsep}{3.8pt}
\caption{Mean ordering accuracy $A_{\mathrm{top}}$ across six linear Gaussian generator regimes ($p=20$; 30 graphs per regime and setting). Population criteria use exact SCM covariances; finite-sample criteria use $n=100$. Higher is better; bold marks the best method within each regime and setting.}
\label{tab:uumc-population-atop}
\begin{tabular*}{\linewidth}{@{\extracolsep{\fill}}lrrrrrr@{}}
\toprule
& \multicolumn{3}{c}{Population $A_{\mathrm{top}}$} & \multicolumn{3}{c}{$n=100$ $A_{\mathrm{top}}$} \\
\cmidrule(lr){2-4}\cmidrule(l){5-7}
Generator & FiCS & VarSort & $R^2$Sort & FiCS & VarSort & $R^2$Sort \\
\midrule
50--50 & \textbf{0.482} & 0.282 & 0.324 & \textbf{0.565} & 0.520 & 0.359 \\
DaO & 0.320 & 0.114 & \textbf{0.535} & 0.353 & 0.293 & \textbf{0.543} \\
IPA & \textbf{0.478} & 0.457 & 0.475 & \textbf{0.510} & 0.476 & 0.481 \\
UUMC & 0.465 & 0.468 & \textbf{0.492} & \textbf{0.493} & 0.471 & 0.482 \\
iSCM & \textbf{0.499} & 0.487 & 0.468 & \textbf{0.526} & 0.510 & 0.482 \\
Unit-variance noise & \textbf{1.000} & 0.944 & 0.611 & \textbf{0.986} & 0.945 & 0.609 \\
\bottomrule
\end{tabular*}
\end{table}

\FloatBarrier

\end{document}